\documentclass{article}
\usepackage{iclr2023_conference,times}
\iclrfinalcopy

\usepackage{hyperref}
\usepackage{commands}
\usepackage{amsmath}
\usepackage{amsthm}
\usepackage{amsfonts}
\usepackage{amssymb}
\usepackage{wrapfig}
\usepackage[inline]{enumitem}
\usepackage{bm}
\usepackage{thm-restate}
\usepackage{nicefrac}
\usepackage{lipsum}
\usepackage{subcaption}
\usepackage[]{framed}

\newtheorem{theorem}{Theorem}
\newtheorem{lemma}[theorem]{Lemma}
\newtheorem{corollary}[theorem]{Corollary}

\theoremstyle{remark}
\newtheorem{remark}{Remark}

\theoremstyle{definition}

\newtheorem{conjecture}{Conjecture}

\usepackage[textsize=scriptsize,textwidth=3cm]{todonotes}

\definecolor{human}{HTML}{1F77B5}
\definecolor{uniform}{HTML}{FF7F0F}
\definecolor{clean}{HTML}{2CA02D}

\newcommand{\rebuttal}[1]{#1}

\newcommand{\reals}{\mathbb{R}}
\newcommand{\br}[1]{\left(#1\right)}

\usepackage{cleveref}

\usepackage{tikz}
\usetikzlibrary{arrows}
\usetikzlibrary{shapes.misc}
\usetikzlibrary{positioning}
\tikzset{cross/.style={cross out, draw, 
         minimum size=4,%
         inner sep=0pt, outer sep=0pt}}
     \definecolor{cadmiumgreen}{rgb}{0.0, 0.42, 0.24}

\graphicspath{{./figures/},{./figures/imagenet_high_inf/}}
\usepackage{pgfplots}
\pgfplotsset{compat=1.11}
\usepgfplotslibrary{fillbetween}
\usetikzlibrary{intersections}
\usetikzlibrary{patterns}
\usetikzlibrary{calc}

\tikzset{
    right angle quadrant/.code={
        \pgfmathsetmacro\quadranta{{1,1,-1,-1}[#1-1]}     %
        \pgfmathsetmacro\quadrantb{{1,-1,-1,1}[#1-1]}},
    right angle quadrant=1, %
    right angle length/.code={\def\rightanglelength{#1}},   %
    right angle length=2ex, %
    right angle symbol/.style n args={3}{
        insert path={
            let \p0 = ($(#1)!(#3)!(#2)$) in     %
                let \p1 = ($(\p0)!\quadranta*\rightanglelength!(#3)$), %
                \p2 = ($(\p0)!\quadrantb*\rightanglelength!(#2)$) in %
                let \p3 = ($(\p1)+(\p2)-(\p0)$) in  %
            (\p1) -- (\p3) -- (\p2)
        }
    }
}
\pgfdeclarelayer{bg}
\pgfsetlayers{bg,main}

\title{A law of adversarial risk, interpolation, and label noise}
\date{}

\author{
  \hspace{-0.2cm}
    Daniel Paleka \thanks{Equal contribution.} \\
	ETH Zurich \\
	\texttt{daniel.paleka@inf.ethz.ch} \\
	\And
	Amartya Sanyal \footnotemark[1] \\
	ETH AI Center,~ETH Zurich \\
	\texttt{amartya.sanyal@ai.ethz.ch}
}

\begin{document}
\setlength{\parindent}{0pt}
\setlength{\parskip}{0.5pc}

\maketitle
\begin{abstract}
  \noindent
In supervised learning, it has been shown that label noise in the data
can be interpolated without penalties on test accuracy.  We show that
interpolating label noise induces adversarial vulnerability, and prove
the first theorem showing the relationship between label noise and
adversarial risk for any data distribution.  
Our results are almost tight if we do not make any assumptions on the
inductive bias of the learning algorithm. We then investigate how
different components of this problem affect this result including
properties of the distribution.
We also discuss non-uniform label noise distributions; and prove a new theorem
showing uniform label noise induces nearly as large an adversarial risk as the
worst poisoning with the same noise rate. 
Then, we provide theoretical and empirical evidence that uniform label noise is more harmful than
typical real-world label noise.  Finally, we show how
inductive biases amplify the effect of label noise and argue the need
for future work in this direction.\looseness=-1

\end{abstract}

\section{Introduction}
\label{sec:introduction}

Label noise is ubiquitous in data collected from the real world. Such
noise can be a result of both malicious intent as well as human error.
The well-known work of \citet{zhang2016understanding} observes that training
overparameterised neural networks with gradient descent can memorize
large amounts of label noise without increased test error.  Recently,
\citet{bartlett2020benign} investigated this phenomenon and termed
it~\emph{benign overfitting}: perfect interpolation of the noisy
training dataset still leads to satisfactory generalization for
overparameterized models. A long series of
works~\citep{donhauser2022fast,hastie2022surprises,muthukumar2020classification}
focused on providing generalization guarantees for models that
interpolate data under uniform label noise. This suggests that
noisy training data does not hurt the test error of overparameterized
models.

However, when deploying machine learning systems in the real world, it is
not enough to guarantee low test error.
Adversarial vulnerability is a practical security
threat~\citep{kurakin2016adversarial,sharif2016accessorize,evtimov2017robust}
for deploying machine learning algorithms in critical environments.
An adversarially vulnerable classifier, that is accurate on the test
distribution, can be forced to err on carefully perturbed inputs even
when the perturbations are small. This has motivated a large
body of work towards improving the \textit{adversarial robustness} of
neural networks~\citep{goodfellow2014explaining,
papernot2016distillation, tramer2018ensemble, sanyal2018robustness,
parseval}. Despite the empirical advances, the theoretical guarantees
on robust defenses are still poorly understood.

Consider the setting of uniformly random label noise. Under certain
distributional assumptions,~\citet{sanyal2020openreview} claim that
with moderate amount of label noise, when training classifiers to zero
training error, the adversarial risk is always large, even when the
test error is low. \rebuttal{Experimentally, this is supported
by~\citet{zhu2021understanding}, who showed that common methods for
reducing adversarial risk like adversarial training in fact does not
memorise label noise.} However, it is not clear whether
their
distributional assumptions are realistic, or if their result is tight.
To deploy machine learning models responsibly, it is important to
understand the extent to which a common phenomenon like label noise
can negatively impact adversarial robustness.  In this work, we
improve upon previous theoretical results, proving a stronger result
on how label noise {\em guarantees} adversarial risk for large enough
sample size. %

On the other hand, existing experimental results~\citep{sanyal2020openreview} 
seem to suggest that neural networks suffer from large adversarial risk 
even in the small sample size regime. 
Our results show that this phenomenon cannot be explained
without further assumptions on the data distributions,
the learning algorithm, or the machine learning model. 
While specific biases of machine learning models and algorithms
~(referred to as inductive bias) have usually played a ``positive'' role in machine
learning literature~\citep{vaswani2017attention,van2017multiscale,mingard2020neural},
we show how some biases can make the model more vulnerable to
adversarial risks under noisy interpolation. %

Apart from the data distribution and the inductive biases, we also investigate the
role of the label noise model.
Uniform label noise, also known as random classification
noise~\citep{angluin1988learning}, is a natural choice for modeling label noise, 
but it is neither the most realistic nor the most adversarial noise model. 
Yet, our results show that when it comes to guaranteeing
a lower bound on adversarial risk for interpolating models, uniform
label noise model is not much weaker than the optimal poisoning adversary. 
Our experiments indicate that natural label noise \citep{wei2022learning}
is not as bad for adversarial robustness as uniform label noise.
Finally, we also attempt to understand the conditions under which such benign~(natural) label noise arises.

\paragraph{Overview }

First, we introduce notation necessary to understand the rest of the paper.  
Then, we prove a theoretical result~(\Cref{thm:adv-risk-lower-bound})
on adversarial risk caused by label noise, significantly improving
upon previous results (\Cref{thm:inflabelballs} from
\citet{sanyal2020openreview}).  In fact, our
\Cref{thm:adv-risk-lower-bound} gives the first theoretical guarantee
that adversarial risk is large for all compactly supported input
distributions and all interpolating classifiers, in the presence of
label noise.  Our theorem does not rely on the particular function
class or the training method.  Then, in \Cref{sec:sample-size}, we
show~\Cref{thm:adv-risk-lower-bound} is tight without further
assumptions, but does not accurately reflect empirical observations on
standard datasets.  Our hypothesis is that the experimentally shown
effect of label noise depends on properties of the distribution and
the inductive bias of the function class.  
In \Cref{sec:non-uniform-label-noise}, we prove (\Cref{thm:uniform-vs-poisoning-informal})
that uniform label noise is on the same order of
harmfmul as worst case data poisoning, given a slight increase in
dataset size and adversarial radius.  
We also run experiments in \Cref{fig:all-exp-human}, showing that
mistakes done by human labelers are more benign than the same rate of
uniform noise.  Finally, in \Cref{sec:wrong-function-class}, we show
that the inductive bias of the function class makes the impact of
label noise on adversarial vulnerability much stronger and provide an 
example in~\Cref{thm:repre-par-inter}.

\section{Guaranteeing adversarial risk for noisy interpolators}
  \label{sec:main-theoretical-results} 

\paragraph{Our setting }
Choose a norm $\norm{\cdot}$ on $\RR^d$, for example $\norm{\cdot}_2$ or $\norm{\cdot}_\infty$.
For $\bm{x} \in \RR^d$, let $\mathcal \mathcal{B}_{r}(\bm{x})$ denote the $\norm{\cdot}$-ball of radius $r$ around $\bm{x}$.
Let $\mu$ be a distribution on $\RR^d$ and let $f^* : \mathcal C \to \{0, 1\}$ be a measurable ground truth classifier. 
Then we can define the adversarial risk of any classifier \(f\) with respect to \(f^*,\mu\), 
given an adversary with perturbation budget \(\rho>0\) under the norm \(\norm{\cdot}\), as

\begin{align}
 \label{eq:adversarial-risk}
  \mathcal R_{\mathrm{Adv}, \rho}(f, \mu) = \PP_{\bm{x} \sim \mu} \left[ \exists \bm{z} \in \mathcal{B}_{\rho}(\bm{x}), ~ f^*(\bm{x}) \neq f(\bm{z})) \right].
\end{align}
  
Next, consider a training set $\left( (\bm{z}_1, y_1), \ldots,
(\bm{z}_m, y_m) \right)$ in $\RR^d \times \{0, 1\}$, where the
$\bm{z}_i$ are independently sampled from $\mu$, and each $y_i$ equals
$f^*(\bm{z}_i)$ with probability $1 - \eta$, where \(\eta>0\) is the
label noise rate. Let $f$ be any classifier which correctly interpolates the training set. 
We now state the main theoretical
result of \citet{sanyal2020openreview} so that we can compare our
result with it.

\begin{restatable}[~\citet{sanyal2020openreview}]{theorem}{inflabelballs}
  \label{thm:inflabelballs}
  Suppose that there exist $c_1 \geq c_2 > 0$, $\rho>0$,
  and a finite set $\zeta \subset \RR^d$ satisfying
  \begin{equation}
    \label{eq:balls_density}
    \mu \left(  {\bigcup_{\bm{s}\in\zeta}\mathcal \mathcal{B}_{\rho/2}(\bm{s})}\right)
    \ge c_1
    \quad \text{and} \quad
    \forall \bm{s}\in\zeta,
    ~\mu \left( \mathcal \mathcal{B}_{\rho/2}(\bm{s}) \right)
    \ge \frac{c_2}{\abs{\zeta}}
  \end{equation}
  Further, suppose that each of these balls contains points from a single class.
  Then for $\delta > 0$, when the number of samples 
  $m\ge\frac{\abs{\zeta}}{\eta c_2} \log(\frac{\abs{\zeta}}{\delta})$, 
  with probability $1 - \delta$
  \begin{align}
    \mathcal R_{\mathrm{Adv}, \rho}(f, \mu) \ge c_1.
  \end{align}
\end{restatable}

This is the first guarantee for adversarial risk caused by label noise
in the literature.  However, \Cref{thm:inflabelballs} has two
extremely strong assumptions:
\begin{itemize}[leftmargin=5mm,nosep]
\item
  The input distribution has mass $c_1$ in a union of balls, each of
  which has mass at least $c_2$;
\item Each ball only contains points from a single class. %
\end{itemize}

It is not clear why such balls would exist for real-world datasets, 
or even MNIST or CIFAR-10. 
\rebuttal{
In \Cref{app:distances}, we give some evidence against the second assumption in particular.
}
In~\Cref{thm:adv-risk-lower-bound}, we remove these
assumptions and show that our guarantees hold for all compactly
supported input distributions, with comparable guarantees on
adversarial risk. %

Let $\mathcal{C}$ be a compact subset of $\RR^d$.
An important quantity in our theorem will be the \emph{covering number} 
$N = N(\rho/2; \mathcal C, \norm{\cdot})$ of $\mathcal C$ in the metric
$\norm{\cdot}$.  
The covering number $N$ is the minimum number of
$\norm{\cdot}$-balls of radius $\rho/2$ such that their union contains
$\mathcal C$. 
For any distribution \(\mu\) on \(\RR^d\), 
denote by \(\mu\left(\mathcal{C}\right) = \PP_{x\sim \mu}\left[x\in \mathcal{C}\right] \) the mass of the
distribution \(\mu\) contained in the compact \(\mathcal{C}\).

\begin{restatable}{theorem}{thmimproved} 
    \label{thm:adv-risk-lower-bound}
    Let $\mathcal{C} \subset \RR^d$ satisfy $\mu(\mathcal{C}) > 0$,
    and let $N = N(\rho/2; \mathcal{C}, \norm{\cdot})$ be its covering number.
    For $\delta > 0$, when the number of samples satisfies
\(m \ge \frac{8N}{\mu(\mathcal C) \eta} \log{\frac{2N}{\delta}}\).
  with probability $1 - \delta$ we have that
  \begin{equation}
    \label{eq:adversarial-risk-lower-bound}
    \mathcal R_{\mathrm{Adv}, \rho}(f, \mu) \ge \frac14 \mu(\mathcal C)
  \end{equation}
  \rebuttal{for any classifier $f$ that interpolates the training set.}
\end{restatable}

The compact $\mathcal{C}$ can be chosen freely, allowing us to make
tradeoffs between the required number of samples $m$ and the lower
bound on the adversarial risk.  As the chosen \(\mathcal{C}\) expands
in volume, the lower bound on the adversarial risk
\(\mu\left(\mathcal{C}\right)\) also increases. However, this also
increases the required number of samples for the theorem to kick in,
which depends on its covering number \(N\).  The tradeoff curve
depends on the distribution $\mu$; we discuss this in
\Cref{sec:sample-size}. 

Note that \Cref{thm:adv-risk-lower-bound} is easier to interpret than \Cref{thm:inflabelballs},
as it holds for any compact \(\mathcal{C}\) 
as opposed to a finite set \(\zeta\) of dense balls.
Our result avoids the unwieldy assumptions, and in fact gives a slightly stronger
guarantee than \Cref{thm:inflabelballs}.  When \Cref{eq:balls_density}
holds, note that we can choose the compact \(\mathcal{C} =
{\bigcup_{\bm{s}\in\zeta}\mathcal \mathcal{B}_{\rho/2}(\bm{s})}\)
from~\Cref{thm:inflabelballs} yielding \(N=\abs{\zeta}\) and
\(\mu\left(\mathcal{C}\right) =c_1\). Thus, under similar settings as
the previous result, our theorem requires the number of samples $m =
\widetilde{\Omega}\left(\frac{\abs{\zeta}}{\eta c_1} \right)$, which is
smaller than $m = \widetilde{\Omega}\left(\frac{\abs{\zeta}}{\eta c_2} \right)$
required in \Cref{thm:inflabelballs}.  

We leave the proof of \Cref{thm:adv-risk-lower-bound} to \Cref{app:proof-of-improved-thm1}, 
but we provide a brief sketch of the ideas. 

\begin{wrapfigure}{r}{0.4\linewidth}
  \centering
  \centering
    \includegraphics[width=0.9\linewidth]{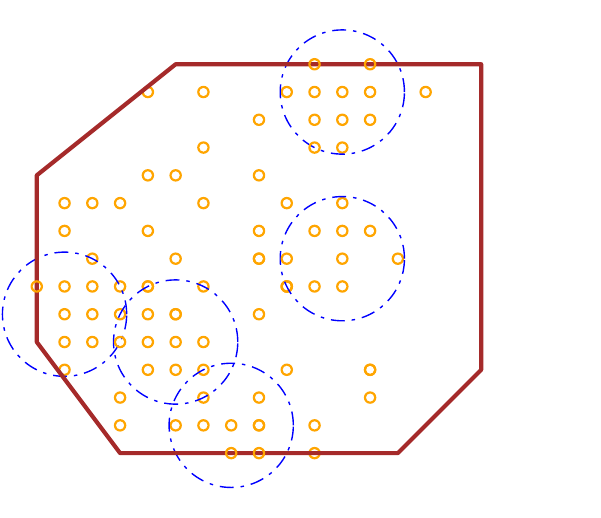} \\ 
  \caption{
    Depending on the covering number of $\mathcal C$, a
    small number of \(\norm{\cdot}\)-balls of sufficient density
    cover a lot of the measure of $\mathcal C$.
    Label noise makes every point drawn from the covered set adversarially vulnerable. 
  }
  \label{fig:illustration-of-subcover}
\end{wrapfigure}
\noindent\textit{Proof sketch}~ We want to prove that a large portion
of points from $\mu$ have a mislabeled point nearby when $m$ is large
enough. The expected number of label noise training points is $\eta
m$; however a priori those could be anywhere in the support of $\mu$.

The key idea is that we can always find a set of $\norm{\cdot}$-balls
covering a lot of measure, with each of the balls having a large
enough density of $\mu$. We prove this in \Cref{lemma:greedy-subcover}
and provide an illustration in~\Cref{fig:illustration-of-subcover}.
The blue dotted circles in~\Cref{fig:illustration-of-subcover} are the
$\norm{\cdot}$-balls; they do not cover the entire space but cover a
significant portion of the entire density. Then, if we take a lot of
$\norm{\cdot}$-balls with large density of a single class, we can
prove that label noise induces an opposite-labeled point in each of
the chosen balls given $m$ large enough.

Concretely, the probability for a single chosen ball to not be adversarially vulnerable is
on the order of $\left(1 - \frac{\eta}{2N} \right)^{\mu(\mathcal{C}) m}$,
and summing this up over the $O(N)$ chosen balls goes to zero when $m$ is large.
By the union bound, each of these balls is then adversarially vulnerable, summing up to a constant adversarial risk.

We considered the binary classification case for
simplicity.  The proof in \Cref{app:proof-of-improved-thm1} lower
bounds the adversarial risk on a single true class.  Thus, by summing
up the risks for each class, we lose only a constant factor on the
guaranteed adversarial risk in the multi-class case. 

For compactly supported $\mu$, 
we can take $\mathcal{C}$ to be the support of $\mu$ to prove a general statement.

\begin{corollary}
  \label{corr:improved-thm1-supp}
    Let $N$ be the covering number of $\supp(\mu)$ with balls of radius $\rho/2$.
    For $\delta > 0$, when the number of samples satisfies
  \(m \ge \frac{8N}{\eta} \log{\frac{2N}{\delta}}\).
  with probability $1 - \delta$ we have that
  \begin{equation}
    \mathcal R_{\mathrm{Adv}, \rho}(f, \mu) \ge \frac14.
  \end{equation}
\end{corollary}
This is easier to understand than \Cref{thm:adv-risk-lower-bound}: 
if interpolating a dataset with label noise, 
the number of samples required to guarantee constant adversarial risk
scales with the covering number of the support of the distribution.

\begin{remark}
The proof in \Cref{app:proof-of-improved-thm1} actually proves \Cref{eq:adversarial-risk-lower-bound}
by first proving a stronger fact:
if $\mu |_{\mathcal{C}}$ is the normalized restricton of $\mu$ on $\mathcal{C}$, then
\begin{equation}
  \mathcal R_{\mathrm{Adv}, \rho}(f, \mu |_{\mathcal{C}}) \ge \frac14.
\end{equation}
We can use this to give better guarantees when $f^*$ is not robust.
If a region of $\supp(\mu)$ is already adversarially vulnerable using the true classifier $f^*$, 
we can omit it from $\mathcal{C}$, and just add the guarantee from \Cref{thm:adv-risk-lower-bound} to the original adversarial risk
to get a stronger lower bound on $\mathcal R_{\mathrm{Adv}, \rho}(f, \mu)$.
\end{remark}

\section{Practical implications on sample size} %
\label{sec:sample-size}
In this section, we discuss the limitations of results like
\Cref{thm:adv-risk-lower-bound}.
When we allow arbitrary interpolating classifiers, we show that
\Cref{thm:adv-risk-lower-bound} paints an accurate picture of the
interaction of label noise, interpolation, and adversarial risk.
However, this particular theoretical framework 
cannot explain the strong effect of label noise on  adversarial risk
in practice~(see \Cref{fig:ind-bias-adv-pic}).  
We argue that this requires a better understanding of the inductive biases of the hypothesis class and the
optimization algorithm.\looseness=-1

\paragraph{Required sample size for~\Cref{thm:adv-risk-lower-bound}}
The number of required samples $m$ in \Cref{thm:adv-risk-lower-bound}
can be very large, depending on the density and the covering number of
the chosen compact $\mathcal C$.  Consider $\norm{\cdot}$ to be the
$\norm{\cdot}_\infty$ norm, as is customary in adversarial robustness
research \citep{goodfellow2014explaining}.  Then the balls $B_{\rho}$
are small hypercubes in $\RR^d$.  If we choose $\mathcal C$ to be the
hypercube $[0, 1]^d$, the covering number scales exponentially in
dimension:
\begin{align}
  N = N(\rho; [0, 1]^d, \norm{\cdot}_\infty) \simeq \left( \frac1{\rho} \right)^d.
\end{align}

\begin{figure*}[t]
  \begin{subfigure}[t]{0.32\linewidth}
    \begin{subfigure}[t]{0.49\linewidth}
      \centering \def\svgwidth{0.99\columnwidth}
      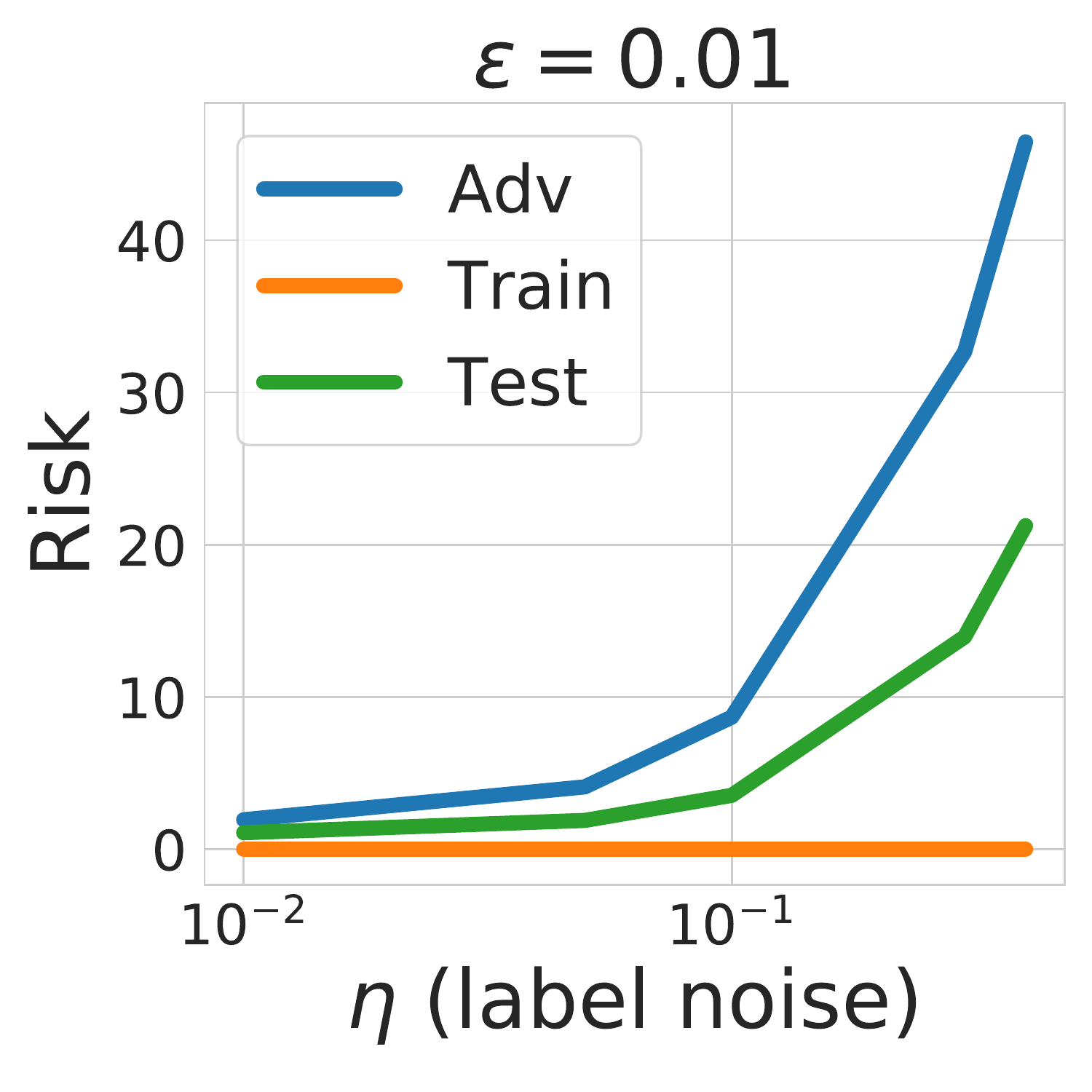
    \end{subfigure}
    \begin{subfigure}[t]{0.49\linewidth}
     \centering \def\svgwidth{0.99\columnwidth}
     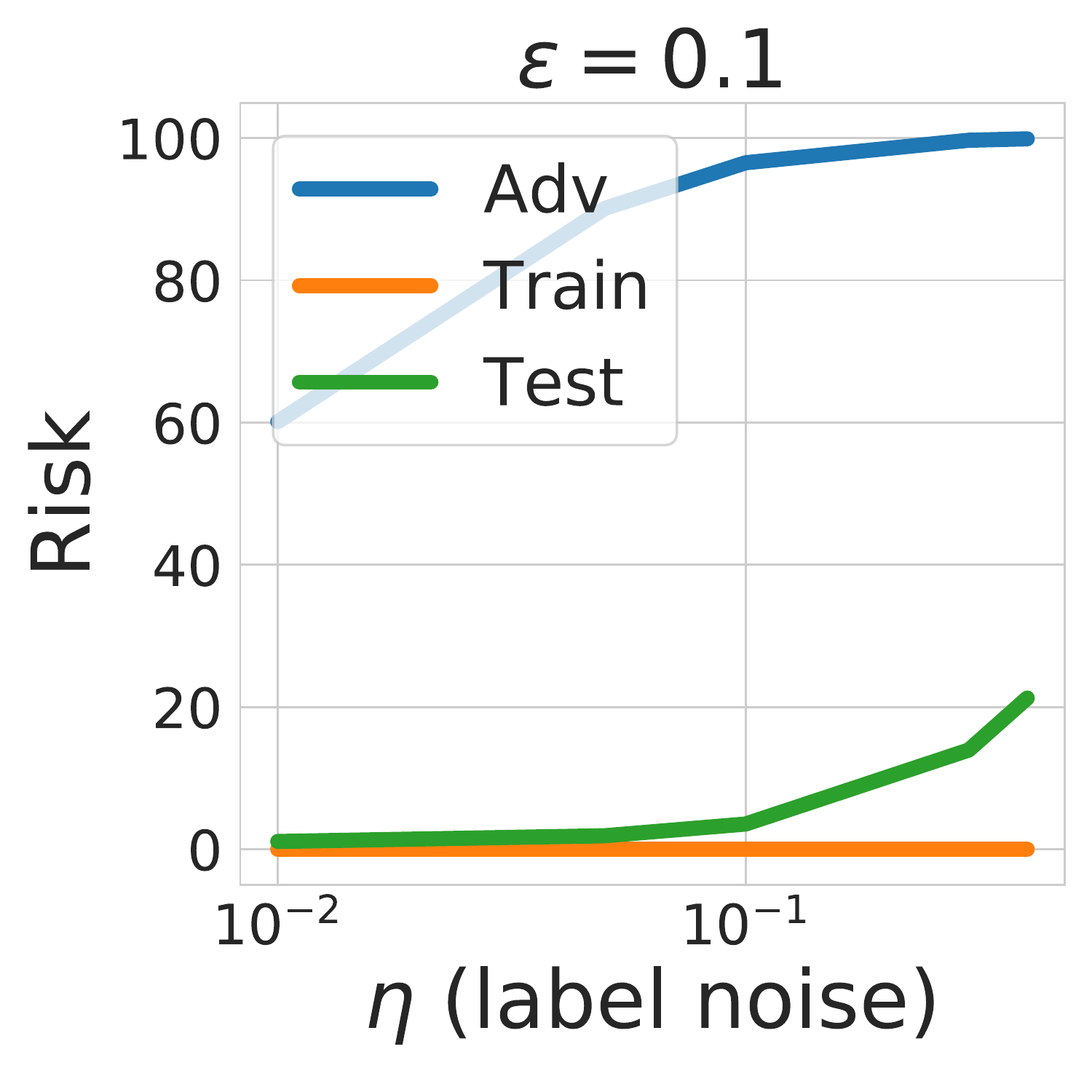
    \end{subfigure}
    \caption{MNIST}
    \label{fig:mnist_lbl_noise_adv}
  \end{subfigure}\vspace{10pt}
    \begin{subfigure}[t]{0.32\linewidth}
      \centering \def\svgwidth{0.99\columnwidth}
      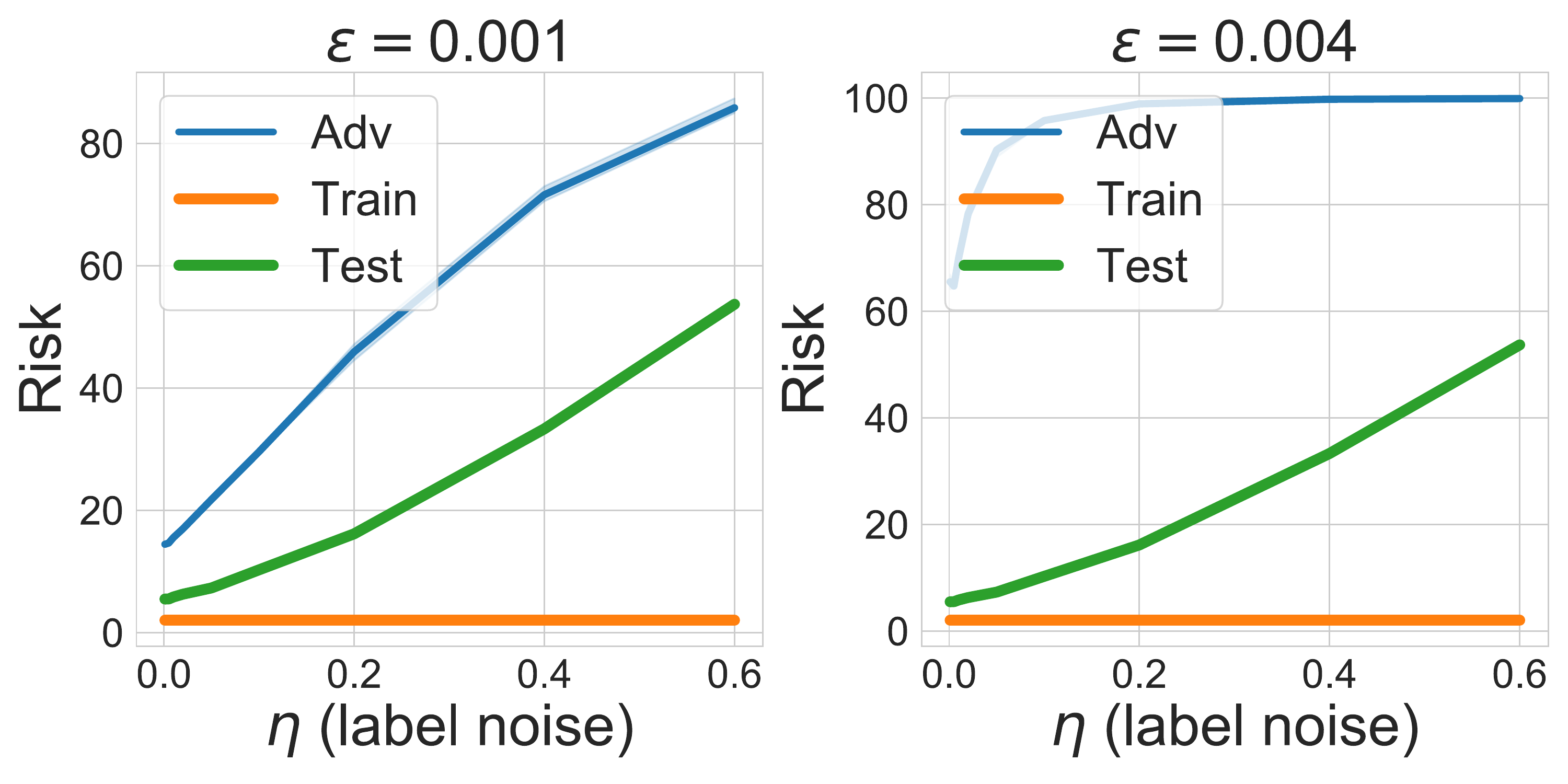
      \caption{ResNet18~(CIFAR10)}
    \label{fig:risk_vs_noise_cifar10_r18}
    \end{subfigure}
    \begin{subfigure}[t]{0.32\linewidth}
      \centering \def\svgwidth{0.99\columnwidth}
      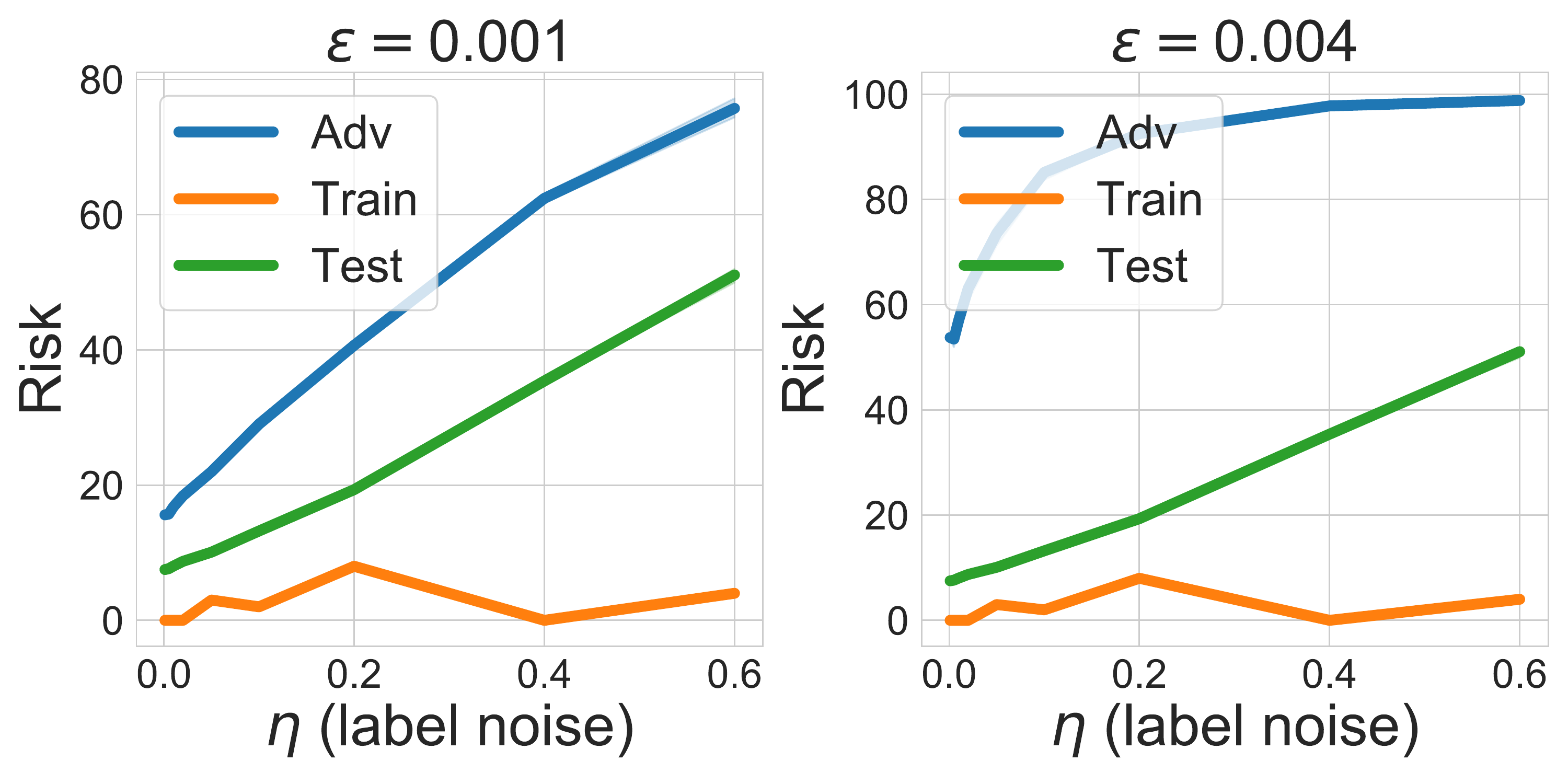
      \caption{DenseNet121~(CIFAR10)}
    \label{fig:risk_vs_noise_cifar10_d121}
    \end{subfigure}
  \caption{From~\citet{sanyal2020openreview}. Adversarial error
  increases with increasing label noise~$\eta$~(x-axis) at a rate much
  faster than predicted by~\Cref{thm:adv-risk-lower-bound}. %
  Here, \(\epsilon\) is the perturbation magnitude~(\(\rho\) in the
  current paper). The label noise is synthetically injected in the
  training set with probability \(\eta\).}
  \label{fig:ind-bias-adv-pic}
\end{figure*}

A rough back-of-the-envelope calculation indicates that this can scale
badly even for standard datasets such as MNIST ($d = 784$) or CIFAR-10
($d = 3072$), since in \Cref{thm:adv-risk-lower-bound} we need $ m
\gtrsim \frac{N}{\mu(\mathcal C) \eta}.$ This amounts an impossibly
large sample size~($m \gtrsim 10^{784}$) for \(\rho=0.2\) to explain
the effect already observed with in $m = 50000$ MNIST training samples
in~\Cref{fig:mnist_lbl_noise_adv}.

Hence our result often does not guarantee any adversarial risk if the
number of samples $m$ is small.  In general, the covering number of a
dataset is not polynomial in the dimension, except if the data has
special properties in the given metric.  For example, if the data
distribution is supported on a subspace of $\RR^d$ of ambient
dimension $k < d$, we can pick a $\mathcal{C}$ for which the covering
number in $\norm{\cdot}_2$ will depend only on $k$ and not on $d$.
However, this is still not sufficient to explain the behaviour
in~\Cref{fig:risk_vs_noise_cifar10_r18,fig:risk_vs_noise_cifar10_d121}.
If our result indeed kicks in for some ambient dimension \(k\) with
DenseNet121 on CIFAR10, then for a given adversarial risk~(say
\(80\%\)), the power law dependence would imply
\(\left(\frac{\rho_2}{\rho_1}\right)^k \approx
\frac{\eta_1}{\eta_2}\), where label noise rate \(\eta_i\) yields
adversarial error \(80\%\) with perturbation budget \(\rho_i\).  With
another back-of-the-envelope calculation
using~\Cref{fig:risk_vs_noise_cifar10_d121}, we set
\(\eta_1,\eta_2=0.7,0.1\) and \(\rho_1,\rho_2 = 0.001,0.004\). This
yields an ambient dimension \(k<2\), which is unrealistic for
CIFAR-10. %

These calculations suggest the possibility that a tighter bound
than~\Cref{thm:adv-risk-lower-bound} might exist.  However, the large
sample size is not just a limitation of
\Cref{thm:adv-risk-lower-bound}. In fact, we show that if arbitrary
classifiers and distributions are allowed, the adversarial risk cannot
be lower bounded for $m = \operatorname{poly}(d)$.

\paragraph{Our result is tight}
It is a priori possible that the true dependence of
adversarial risk on label noise kicks in for much lower sample size
regimes than in \Cref{thm:adv-risk-lower-bound}. This might suggest
that the lower bound on sample complexity can be improved.  We show this is not the case and in fact our bound is sharp.  
In particular, we design a simple distribution on $\RR^d$ such that
there exist classifiers which correctly and robustly interpolate
datasets with the number of samples $m$ exponential in
$d$.
\begin{restatable}{proposition}{spherelabelnoise} \label{prop:sphere-label-noise}
  Let $\mu$ be the uniform distribution on 
  $\sph^{d-1} = \left\{ 
        x_1, \ldots, x_d \in \RR^d : x_1^2 + \ldots + x_d^2 = 1 
    \right\} $,
  and let the ground truth classifier $f^*$ be a threshold function on $x_1$:
  $f^*(x) = \one_{x_1 > \frac12}$.
  Consider any adversarial radius $\rho < \frac14$ in the Euclidean metric.
  Then, for any label noise $\eta < 1$:
  with high probability,
  there exists a classifier $f$ that interpolates $m = \lfloor 1.01^d \rfloor$
  samples from the label noise distribution, 
  such that $\mathcal R_{\mathrm{Adv}, \rho}(f, \mu) = o_d(1)$.
\end{restatable}

\noindent\textit{Proof sketch}~
The main ingredient of the proof is the concentration of measure on $\sph^{d-1}$,
which makes the training samples far apart in the Euclidean metric.
We leave the full proof to \Cref{app:proof-sphere-label-noise}.
Similar statements in the clean data setting have appeared before,
e.g. in \cite{bubeck2021universal}.

Note that \Cref{prop:sphere-label-noise} shows a construction where
\Cref{thm:adv-risk-lower-bound} cannot guarantee an adversarial risk
lower bound with sample size \(m\) sub-exponential in \(d\).
Hence, the covering number of any substantial portion of $\sph^{d-1}$ is
exponential in the dimension $d$. This unintentionally proves the well
known fact that the covering number of the sphere $\sph^{d-1}$ in the
Euclidean metric is exponential.\footnote{See Proposition 4.16
in~\url{https://www.stats.ox.ac.uk/~rebeschi/teaching/AFoL/20/material/lecture04.pdf}}

\paragraph{Optimizing $\mathcal{C}$ can avoid large sample size}
While~\Cref{prop:sphere-label-noise} shows the tightness
of~\Cref{thm:adv-risk-lower-bound} in the worst case, it is
possible a smaller sample size requirement is sufficient under certain
conditions.  
In particular, if we can pick a compact $\mathcal C$ with
small covering number, such that the measure $\mu(\mathcal C)$ is
large, then~\Cref{thm:adv-risk-lower-bound} allows for a small sample
size while guaranteeing a large adversarial risk. %

\noindent\textit{Example}~ Take an adversarial radius $\rho > 0$ in
the $\norm{\cdot}_\infty$ metric and choose \(r\in(0, \frac{1}{2})\),
let $\mu = (1-r) \mu_1 + r \mu_2$ be the average of two
measures, $\mu_1$ and $\mu_2$, with $\mu_1$ the uniform distribution
on $[0, 1]^d$, and $\mu_2$ the uniform distribution on a smaller
hypercube $[0, \rho]^d$.

The first choice $\mathcal{C} = [0, 1]^d$ as in
\Cref{corr:improved-thm1-supp} has covering number on the order of
$\rho^{-d}$.  \Cref{thm:adv-risk-lower-bound} is then vacuous until $m
\gtrsim \rho^{-d}/\eta$, which is very large in high dimensions.  Note
that this is necessary to achieve the lower bound of adversarial risk
of \(\frac{1}{4}\). However, if we only want to guarantee an
adversarial risk of \(\frac{r}{4}\), instead we can use $\mathcal{C} =
[0, \rho]^d$. For this, the covering number is $1$ and we can use
\Cref{thm:adv-risk-lower-bound} for $m = O\br{\frac1{\eta}}$. This
suggests that while the required sample size for the maximal
adversarial risk is possibly very large, it can be much smaller,
depending on the distribution, for guaranteeing a smaller adversarial
risk.

Formally, to get the ``best possible'' $m$ in \Cref{thm:adv-risk-lower-bound} for a certain adversarial risk lower bound $r$,
we should solve the following optimization problem over subsets of $\supp(\mu)$:
\begin{align}
  \label{eq:optimize-subset}
  \min_{\mu(\mathcal{C}) \ge 4r} \frac{N(\rho/2, \mathcal{C}, \norm{\cdot}_\infty) \log N(\rho/2, \mathcal{C}, \norm{\cdot}_\infty)}{\mu(\mathcal{C})}.
\end{align}
The above optimization problem comes from substituting \(r\) into the
adversarial risk placeholder in \Cref{thm:adv-risk-lower-bound}.
\Cref{eq:optimize-subset} provides a complexity measure to get tighter
lower bounds on the adversarial vulnerability induced by uniform label
noise. It is not known whether the optimization is tractable in
general.  However, the concept of having to solve an optimization
problem in order to get a tight lower bound is common in the
literature.  Some examples are the \emph{representation
dimension}~\citep{beimel2019characterizing} and the \emph{SQ
dimension}~\citep{feldman2017general} .

To conclude, this section shows that in real world data, the required
sample size for guaranteeing large adversarial risk from interpolating
label noise is significantly smaller than what an off-the-shelf
application of~\Cref{thm:adv-risk-lower-bound} might suggest.
However, we also proved that it is not possible to obtain tighter
bounds without further assumptions on the data or the model.

\section{Non-uniform label noise}
\label{sec:non-uniform-label-noise}

In previous sections, we discussed guaranteeing a lower bound on
adversarial error for noisy interpolators
in~\Cref{sec:main-theoretical-results}. In~\Cref{sec:sample-size}, we
discussed the tightness of the said bound. However, all of these
results assumed that the label noise is distributed uniformly on the
points in the training set, which corresponds to the popular Random
Classification Noise~\citep{angluin1988learning} model. However, an
uniform noise model is not very
realistic~\citep{hedderich2021analysing,wei2022learning}; and it is
thus sensible to also investigate how our results change under
non-uniform label noise models.

\paragraph{Uniform noise is almost as harmful as poisoning}

The worst-case non-uniform label noise model is \emph{data poisoning}, 
where an adversary can choose the labels of a subset of the training set of a fixed size~(see~\citet{biggio2018wild} for a survey). 
It is well known that flipping the label of a constant number of points in the training set
can significantly increase the error of a logistic regression model
\citep{jagielski2018manipulating} or an SVM classifier
\citep{biggio2011support}.  On the contrary, the test error of a
neural networks has been surprisingly difficult to hurt by data
poisoning attacks which flip a small fraction of the
labels.~\citet{lu2022indiscriminate} show that, on some datasets, the
effect of adversarial data poisoning on test accuracy is, in fact,
comparable to the effect of uniform label noise. 

We phrase the main result of this section informally
in~\Cref{thm:uniform-vs-poisoning-informal}, using standard
game-theoretic metaphors for data poisoning, and defer the formal
version to~\Cref{app:poisoning-theorem}.  Let again $\mu$ be a
distribution on $\RR^d$ and $f^*$ a correct binary classifier, and let
$\eta$ be the label noise rate.  Consider a game in which an adversary
flips the labels of a subset of the training set, and tries to
maximize the minimum adversarial risk among all interpolators of the
noisy~(after flipping labels) training set. We will compare the performances
of two adversaries:
\begin{itemize}
  \item \textbf{Uniform}, who samples \(T\) points uniformly from the
  distribution, and flips the label of each of the $T$ points in the
  sampled training set with probability $\eta$;
  \item \textbf{Poisoner}, who inserts $N = \eta m$ arbitrary points
  from $\supp(\mu)$ with flipped labels into the training set and then
  samples the remaining \(m - N\) points uniformly, with correct labels.
\end{itemize}

Here $T$ and $m$ are the respective training set sizes.  If $T \sim
\frac1\eta N = m$, then the two adversaries flip the same number of
labels in expectation. In that sense, both of these adversaries have
the same budget. However, the Poisoner can choose which points to flip
and thus intuitively, in this regime, the Poisoner will get a higher
adversarial risk than the Uniform. Surprisingly, we can prove the
Uniform is not much worse if $T \sim m \log m$.  

\begin{theorem}[Informal statement of \Cref{thm:uniform-vs-poisoning}]
  \label{thm:uniform-vs-poisoning-informal}
  Denote the adversarial risks of the Uniform and the Poisoner adversaries 
  by $\mathcal{R}^{\textrm{Unif}}$ and $\mathcal{R}^{\textrm{Poison}}$ respectively. 
  For any $\rho > 0$, we
  have that
  \begin{align}
    \mathcal{R}^{\textrm{Unif}}_{2 \rho} \geq \frac12 \mathcal{R}^{\textrm{Poison}}_{\rho}
  \end{align}
  as long as $\mathcal{R}^{\textrm{Poison}}_{\rho}=\Omega(1)$ and $T \gtrsim m \log m $.
\end{theorem}
Roughly speaking, the above theorem shows that if the Uniform
adversary is given \emph{double the adversarial radius} and \emph{a
log factor increase on the training set size}, then Uniform can
guarantee an adversarial risk of the same magnitude as the Poisoner.
The full statement and the proof of the theorem are given in \Cref{app:poisoning-theorem} but we provide a brief sketch here.

\noindent\textit{Proof sketch}~ 
The Poisoner will choose $N$ points to flip, adversarially poisoning every point in the $N$ corresponding $\rho$-balls. 
As in \Cref{thm:adv-risk-lower-bound}, we can use \Cref{lemma:greedy-subcover} to show that a subset of the balls with density $\Omega(1/N)$
covers half of the adversarially vulnerable region.
Then Uniform samples $T$ points, and we expect to hit each of the balls in the chosen subset. 
Because of the doubled radius, each sampled point makes the the whole $\rho$-ball vulnerable.
The log factor comes from the same reason as in the standard coupon collector (balls and bins) 
problem; if we have $N$ bins with hitting probabilities $\Omega(1/N)$, then we need $\Omega(N \log N)$ tries to hit each bin at least once.

\paragraph{Some label noise models are benign}

Different label noise models with the same expected label noise rate
can have very different effects on the adversarial risk.  In the
previous sections, we showed that uniform label noise is almost as bad
as the worst possible noise model with the same label noise rate. This
raises the question whether all noise models are as harmful as the
uniform label noise model. We answer the question in the negative
especially for data distributions that have a {\em long tailed}
structure: many far-apart low-density subpopulations in the support of the distribution $\mu$.  \looseness=-1

For this, we show a simple data distribution $\mu$
in~\Cref{prop:benign-label-noise}, where:
\begin{itemize}%
  \item Uniform label noise with probability $\eta$ guarantees adversarial risk on the order of $\eta$;
  \item A different label noise model, with expected label noise rate
  \(\eta\), which affects only \emph{the long tail} of the
  distribution $\mu$ can be interpolated with $o(1)$ adversarial risk.
\end{itemize}

We argue that this is neither an unrealistic distributional assumption
nor an impractical noise model. In fact, most standard image datasets,
like SUN~\citep{xiao2010sun}, PASCAL~\citep{everingham2010pascal}, and
CelebA~\citep{liu2015faceattributes} have a long
tail~\citep{zhu2014capturing,sanyal22a}.  Moreover, it is natural to
assume that mistakes in the datasets are more likely to occur in the
long tail, where the data points are atypical.  In
\citet{feldman2020does}, it was argued that noisy labels on the long
tail are one of the reasons for why overparameterized neural networks
remember the training data.  Formally, we prove the following
regarding the benign noise model for a long-tailed distribution.

\begin{restatable}{proposition}{benignlabelnoise} \label{prop:benign-label-noise}
  Let $A < B$ be integers with $A$ much smaller than $B$. 
  Let $\mu$ be a mixture model on $\RR$ supported on a disjoint union of $A + B$ intervals,
  such that half of the mass is on the first $A$ intervals and half of the mass is on the last $B$ intervals:
  \begin{align*}
    \mu &= \frac1{2A} \sum_{i=1}^A \mathrm{Unif}\left(i, i+\frac12 \right) + \frac1{2B} \sum_{j=1}^B \mathrm{Unif}\left(A+j, A+j+\frac12 \right) 
  \end{align*}
  Let the ground truth label be zero \rebuttal{everywhere}.  Sample two datasets
  $\SD_1, \SD_2$ of size $m$ from $\mu$ using two different label
  noise distributions: For $\SD_1$, flip the label of each sample $x
  \in [0, A + B]$ independently with probability $\eta$. For $\SD_2$,
  flip the label of each sample $x \in [A, A + B]$ independently with
  probability $2\eta$, and leave the labels of the other samples
  unchanged.
  Then, for any $\rho,\delta \in \left(0, \frac12 \right)$, for the
  number of samples $m = \tilde{\Theta}_{\rho}(A)$~(ignoring log
  terms), we have that with probability \(1-\delta\):
  \begin{itemize}
    \item For any $f$ which interpolates $\SD_1$, the adversarial risk is large: $\mathcal R_{\mathrm{Adv}, \rho}(f, \mu) = \rebuttal{\Omega_{\rho}} \left( 1 \right)$.
    \item There exists $f$ which interpolates $\SD_2$, such that $\mathcal R_{\mathrm{Adv}, \rho}(f, \mu) = \rebuttal{O_{\rho}} \left( \frac{ A}{B} \right)$.
  \end{itemize}
\end{restatable}

A similar distribution was previously used as a representative long
tailed distribution in the context of privacy and fairness
in~\citet{sanyal22a}. Our result can also be extended to more
complicated long-tailed distributions with a similar strategy. 
~\Cref{prop:benign-label-noise} implies that the the
first noise model~(for \(\mathcal{D}_1\)) induces \(\Omega(1)\)
adversarial risk on all interpolators. On the other hand, for the
second noise model i.e. for \(\mathcal{D}_2\), it is possible to
obtain interpolators with adversarial risk on the order of
\(\left(\frac{ A}{B} \right)\).  Thus, for distributions where \(A\ll
B\), this implies the existence of almost robust interpolators despite
having the same label noise rate.%

\paragraph{Real-world noise is more benign than uniform label noise}
To support our argument that real world noise models are, in fact,
more benign than uniform noise models, we consider the noise induced
by human annotators in \citet{wei2022learning}. They propose a new
version of the CIFAR10/100 dataset~\citep{krizhevsky2009learning}
where each image is labelled by three human annotators. Known as
CIFAR-10/100-n, each example's label is decided by a majority vote on
the three annotated labels. We train ResNet34 models till
interpolation on these two datasets. The label noise rate, after the
majority vote, is \(\approx 9\%\) in CIFAR10 and \(\approx 40\%\) in
CIFAR100. We repeat the same experiment for uniform label noise with
the same noise rates, and also without any label noise.  Each of these
models' adversarial error is evaluated with an \(\ell_\infty\) PGD
adversary  plotted
in~\Cref{fig:human-uniform-clean-10,fig:human-uniform-clean-100}.

\Cref{fig:human-uniform-clean-10,fig:human-uniform-clean-100} show
that, for both CIFAR10 and CIFAR100, uniform label noise is indeed
worse for adversarial risk than human-generated label noise. For
CIFAR-10, the model that interpolates human-generated label noise is
almost as robust as the model trained on clean data.  This supports
our argument that real-world label noise is more benign, for
adversarial risk, than uniform label noise.\looseness=-1

An important direction for future research is understanding what types of
label noise models are useful mathematical proxies for realistic label
noise. We shed some light on this question using the idea of {\em
memorisation score}~\citep{feldman2020neural}.  Informally,
memorisation score quantifies the atypicality of a sample; it measures
the increase in the loss on a data point when the learning algorithm
does not observe it during training compared to when it does. A high
memorisation score indicates that the point is unique in the training
data and thus, likely, lies in the long tail of the distribution.
In~\Cref{fig:mem-score}, we plot the average memorisation score of
each class of CIFAR10 in brown, and the average for images that were
mislabeled by the human annotator in blue.  It is clearly evident
that the mislabeled images have a higher memorisation score. This
supports our hypothesis~(also in~\citet{feldman2020does}) that, in the
real world, examples in the long tail, are more likely to be
mislabeled.

\begin{figure}[t]
  \begin{subfigure}{0.26\linewidth}
    \includegraphics[width=0.99\linewidth]{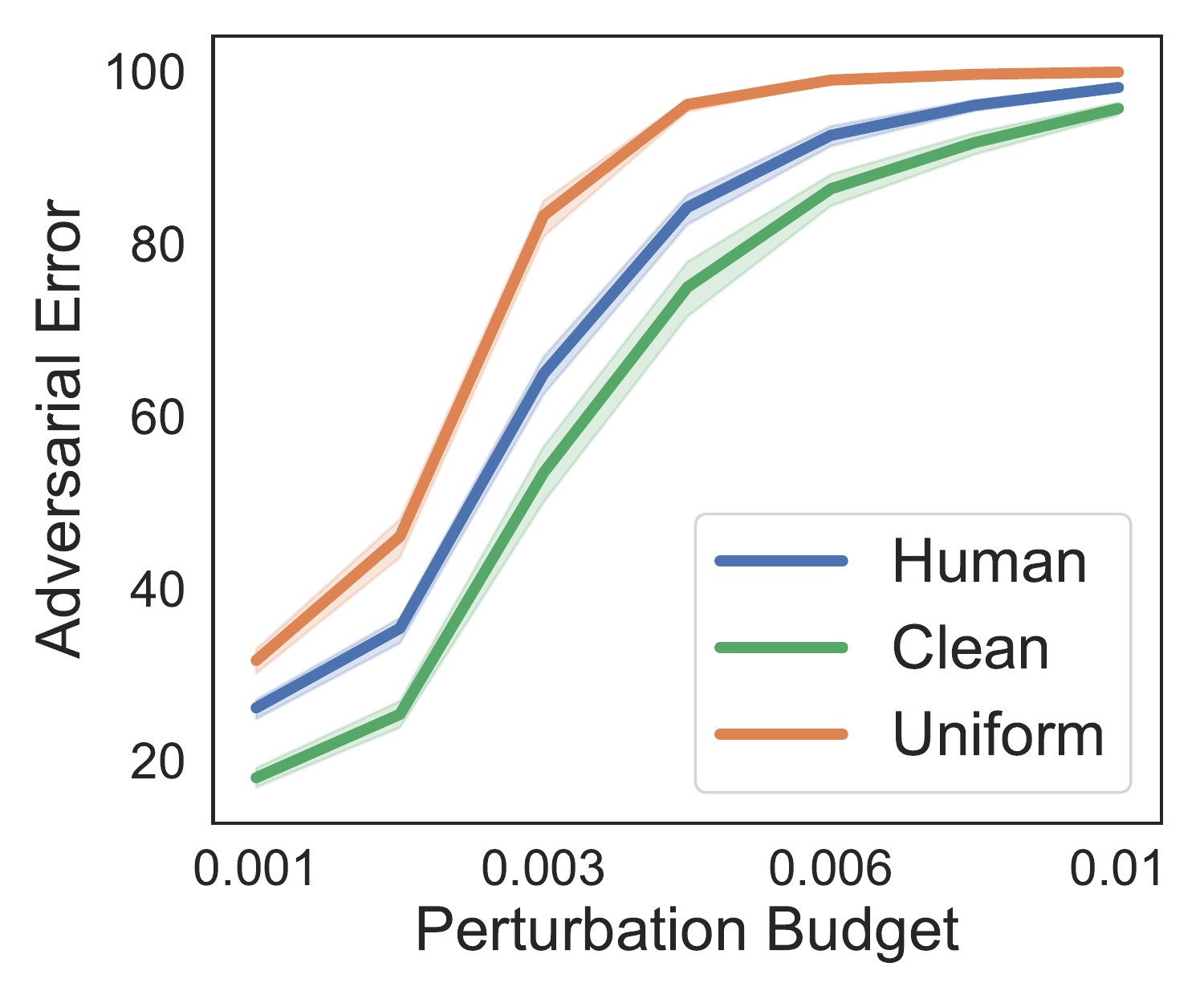}
    \subcaption{CIFAR10}
    \label{fig:human-uniform-clean-10}
  \end{subfigure}
  \begin{subfigure}{0.26\linewidth}
    \includegraphics[width=0.99\linewidth]{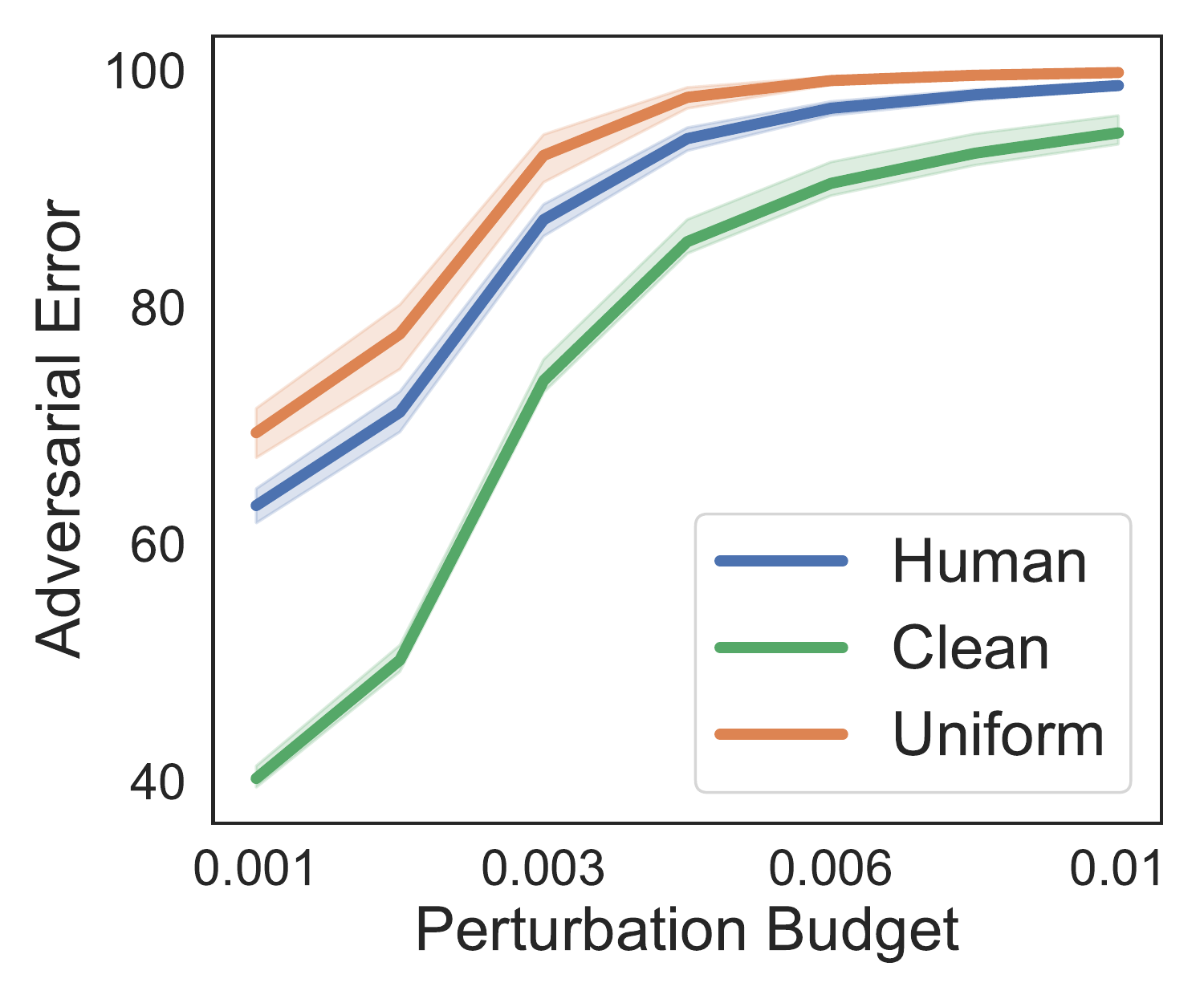}
    \subcaption{CIFAR100}
    \label{fig:human-uniform-clean-100}
  \end{subfigure}
  \begin{subfigure}{0.47\linewidth}
    \includegraphics[width=0.99\linewidth]{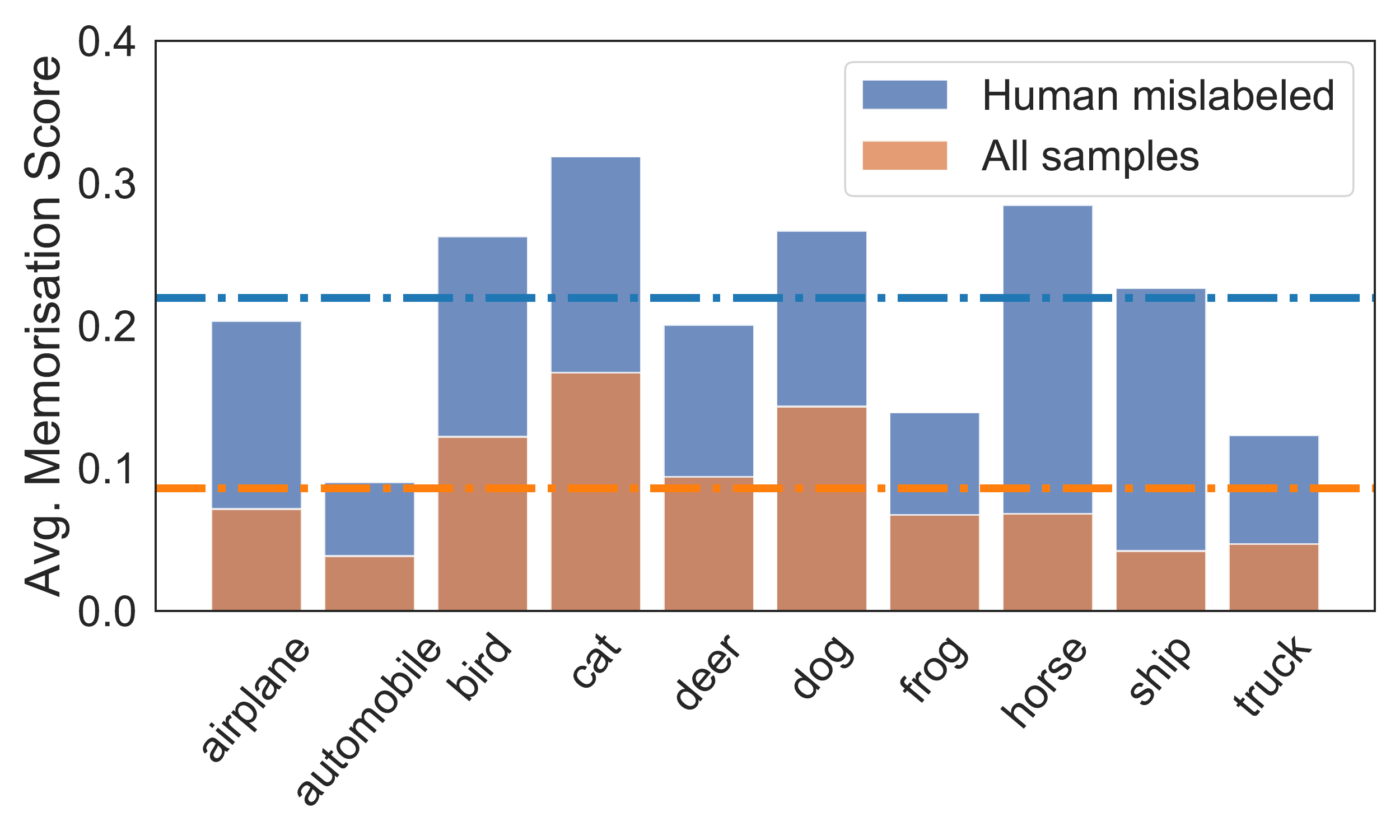} \vspace{-3em}
    \subcaption{CIFAR10}
    \label{fig:mem-score}
  \end{subfigure}
  \caption{\Cref{fig:human-uniform-clean-10,fig:human-uniform-clean-100}
  plots adversarial risk against perturbation budget on CIFAR10 and
  CIFAR100 datasets respectively, with three label noise models.
  {\color{clean} Clean} denotes the original CIFAR-10/100 labels.
  {\color{human} Human} denotes the human-generated labels from
  \citep{wei2022learning}.  {\color{uniform} Uniform} refers to
  uniformly random label noise with the same rate as {\color{human}
  Human}.  \Cref{fig:mem-score} plots the average memorisation score
  per class for: {\color{human} Human mislabeled} examples i.e.
  examples with {\color{human} Human} noisy labels and
  {\color{uniform} All} samples.  The horizontal line is the average
  across all classes.  The higher score of human mislabeled examples
  indicates that those examples belong to the long tail of the
  distribution.\looseness=-1}
  \label{fig:all-exp-human}
\end{figure}

\section{The role of inductive bias}
\label{sec:wrong-function-class}
We have seen, in~\Cref{sec:sample-size}, that without further
assumptions the theoretical guarantees in
\Cref{thm:adv-risk-lower-bound} only hold for very large training
sets. In this section, we discuss how the inductive bias of the
hypothesis class or the learning algorithm can lower the sample size
requirement and point to recent work that indicates the presence of
such inductive biases in neural networks.

\paragraph{Inductive bias can hurt robustness even further}
 There is ample empirical
 evidence~\citep{ortiz2020neural,kalimeris2019sgd,shamir2021dimpled}
 that neural networks exhibit an inductive bias that
 is different from what is required for robustness. %
 ~\citet{shah2020pitfalls} also provides empirical evidence that
 neural networks exhibit a certain inductive bias, that they call
 simplicity bias, that hurts adversarial
 robustness.~\citet{ortiz2022catastrophic} show that this is also
 responsible for a phenomenon known as catastrophic overfitting. 

 Here, we show a simple example to illustrate the role of inductive bias.
 Consider a binary classification problem on a data distribution
 \(\mu\) and a dataset \(S_{m,\eta}\) of $m$ points, sampled i.i.d. from
 \(\mu\) such that the label of each example is flipped with
 probability \(\eta\). 

\begin{restatable}{theorem}{thmrepreparinter}
  \label{thm:repre-par-inter} 
For any \(\rho>0\), there
exists a distribution \(\mu\) on $\RR^2$ and two hypothesis classes~\(\SH\) and
\(\SF\), 	such that for any label noise rate
	\(\eta\in\br{0,\nicefrac{1}{2}}\) and dataset size
	\(m=\Theta\br{\frac{1}{\eta}}\), in expectation we have that:
for all \(h\in\mathcal{H}\) that interpolate \(S_{m,\eta}\), 
\begin{align}
\mathcal{R}_{\mathrm{Adv},\rho}\br{h, \mu} \geq\Omega\br{1} ;
\end{align}
whereas there exists an \(f \in \mathcal{F}\) that interpolates
\(S_{m,\eta}\) and
\(\mathcal{R}_{\mathrm{Adv},\rho}\br{f; \mu} =\mathcal{O}\br{\rho}\).
\end{restatable}

\rebuttal{ The classes \(\SH\) and \(\SF\) are precisely defined in
\Cref{app:ind-bias} where a formal proof is given as well}, but we
provide a proof sketch here.  The data distribution \(\mu\)
in~\Cref{thm:repre-par-inter} is uniform on the set~\([0,
W]\times\{0\}\), that is, the data is just supported on the first
coordinate where \(W\gg \rho>0\).  
The ground truth is a threshold function on the first coordinate.
\rebuttal{The constructed \(f\in\SF\) simply labels everything
according to the ground truth classifier~(which is a threshold
function on the first coordinate) except the mislabeled data points;
where it constructs infinitesimally small intervals around the point
on the first coordinate}. Note that this construction is similar to
the one in~\Cref{prop:sphere-label-noise}. By design, it interpolates
the training set and its expected adversarial risk is upper bounded by
\(2m\eta\rho\). \looseness=-1

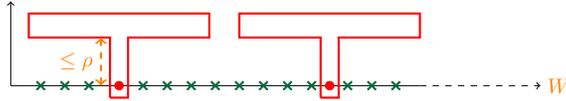
\begin{figure}[t]\centering
      \scalebox{0.8}{\begin{tikzpicture}
        \draw[] (1.2,0) -- (8,0);\draw[dashed, ->] (8,0) -- (10,0)
        node[right, orange]{\(W\)}; \filldraw [red] (3,0) circle
        (2pt);
        \filldraw [red] (6.5,0) circle (2pt);
        \draw[->] (1.2,0) -- (1.2,1+0.4);

        \draw[-, red, line width=1pt] (2.85,-0.2) --
        (3.15,-0.2) -- (3.15, 0.4+0.4) -- (4.5, 0.4+0.4) -- (4.5, 0.8+0.4) -- (1.5,0.8+0.4) -- (1.5, 0.4+0.4) -- (2.85, 0.4+0.4) -- (2.85, -0.2); 
        \draw[dashed, <->, orange, line width=1pt] (2.7, 0) -- node[midway,left] {\(\leq\rho\)}  (2.7, 0.8);
        \draw[-, red, line width=1pt] (2.85+3.5,-0.2) --
        (3.15+3.5,-0.2) -- (3.15+3.5, 0.4+0.4) -- (4.5+3.5, 0.4+0.4) -- (4.5+3.5, 0.8+0.4) --
        (1.5+3.5,0.8+0.4) -- (1.5+3.5, 0.4+0.4) -- (2.85+3.5, 0.4+0.4) -- (2.85+3.5, -0.2); 
        \draw (1.7,0 )
      node[cross,cadmiumgreen,line width=1pt] {}; \draw (2.1,0 )
      node[cross,cadmiumgreen,line width=1pt] {}; \draw (2.5,0 )
      node[cross,cadmiumgreen,line width=1pt] {}; \draw (3.4,0 )
      node[cross,cadmiumgreen,line width=1pt] {}; \draw (3.8,0 )
      node[cross,cadmiumgreen,line width=1pt] {}; \draw (4.2,0 )
      node[cross,cadmiumgreen,line width=1pt] {}; \draw (2.2+2.4,0 )
      node[cross,cadmiumgreen,line width=1pt] {}; \draw (2.6+2.4,0 )
      node[cross,cadmiumgreen,line width=1pt] {}; \draw (3.+2.4,0 )
      node[cross,cadmiumgreen,line width=1pt] {}; \draw (3.4+2.4,0 )
      node[cross,cadmiumgreen,line width=1pt] {}; \draw (3.4+2.4+0.4,0 )
      node[cross,cadmiumgreen,line width=1pt] {}; \draw (3.4+2.4+1,0 )
      node[cross,cadmiumgreen,line width=1pt] {}; \draw (3.4+2.4+1.4,0 )
      node[cross,cadmiumgreen,line width=1pt] {}; \draw (3.4+2.4+1.8,0 )
      node[cross,cadmiumgreen,line width=1pt] {};

    \end{tikzpicture}}%
          \caption{Visualization of a portion of the distribution
          \(\mu\) and the hypothesis class \(\SH\) used in~\Cref{thm:repre-par-inter}.~The crosses are the
          mislabeled examples and the circles are correctly labelled
          examples. All the circles are adversarially
          vulnerable to upwards perturbations of magnitude less than \(\rho\).} 
          \label{fig:t-shape-region}%
        \end{figure}

\rebuttal{Each hypothesis in the hypothesis class \(\SH\) can be
thought of as a union of T-shaped decision regions.  The region inside
the T-shaped regions are classified as \(1\) and the rest as \(0\).}
Note that the ``head'' of a T-shaped region make the region on the
data manifold~(first coordinate) directly below them adversarially
vulnerable. \looseness=-1
The width of the T can be interpreted as the inductive bias of the learning algorithm. 
The decision boundaries of neural networks usually lie on the data manifold 
~\citep{somepalli2022can}; and the network behaves more smoothly off the data manifold. 
A natural consequence of this is that the head of the Ts are large. 
This is not the exact explanation of inductive bias in neural networks,
but rather an illustrative example for what might be happening in practice.
\rebuttal{In~\Cref{app:exp-ind-bias}, we provide
experimental evidence to show this type of behaviour for neural networks.}\looseness=-1

There are two important properties of this simple example relevant for
understanding adversarial vulnerability of neural networks.  First,
the adversarial examples constructed here are off-manifold: they do
not lie on the manifold of the data. This has been observed in prior
works~\citep{hendrycks2016baseline,sanyal2018robustness,
khoury2018geometry}.  Second, our examples implicitly exhibit the {\em
dimpled manifold} phenomenon recently described
in~\citet{shamir2021dimpled}. \looseness=-1

\paragraph{Is \Cref{thm:adv-risk-lower-bound} about the wrong function
class?} When fitting deep neural networks to real datasets, the
results of \Cref{thm:adv-risk-lower-bound} still hold even when the
number of samples $m$ is much smaller than required, as can be seen in
\Cref{fig:ind-bias-adv-pic}. 
We think that proving guarantees on adversarial risk in the presence of label noise is within reach for simple neural network settings. 
Towards this goal, we propose a conjecture in a similar vein to \cite{bubeck2020law}:

\begin{conjecture}
  Let $f : \RR^d \to \RR$ be a neural network with a single hidden layer with $k$ neurons.
  Under the same conditions as in \Cref{thm:adv-risk-lower-bound},
  for the number of samples $m = \widetilde{\Omega}(\frac1\eta \textrm{poly}(k, d))$,
  \begin{align}
    \mathcal R_{\mathrm{Adv}, \rho}(f, \mu) \ge \text{const.}
  \end{align}
  for a distribution $\mu$ supported on $[0, 1]^d$.
\end{conjecture}

In short, we conjecture that neural networks exhibit inductive biases
which hurt robustness when interpolating label noise.
\citet{dohmatob2021tradeoffs} show a similar result for a large class of neural networks. 
However, the Bayes optimal error in their setting is a positive constant (as opposed to zero in our setting), and we assume uniform label noise in the training set. 
Understanding these properties is important for training on real-world data, 
where label noise is not a possibility but rather a norm. \looseness=-1

\newpage
\subsubsection*{Acknowledgments}
Amartya Sanyal acknowledges the ETH AI Center for the postdoctoral fellowship.
We thank Mislav Balunović for discussing and checking the proof of \Cref{thm:adv-risk-lower-bound},
Domagoj Bradač for discussing \Cref{lemma:greedy-subcover}, and Jacob Clarysse and Florian Tram{\`e}r for general feedback.

\bibliographystyle{plainnat}
\bibliography{references}

\appendix
\crefalias{section}{appendix}
\onecolumn

\section{Proof of \Cref{thm:adv-risk-lower-bound}}
\label{app:proof-of-improved-thm1}

Here we prove the following statement:

\thmimproved*

For notational convenience, we replace $\rho$ by $2 \rho$ in all places for the proof below.

\begin{proof}
  Without loss of generality, let
  $ \mathcal C_0 = \set{ \bm{x} \in \SC : f^*(\bm{x}) = 0} $ have probability
  $\mu(\mathcal C_0) \ge \frac12 \mu(\mathcal C)$. Let $\mu_0 = \mu |_{\mathcal C_0}$, normalized so that
  $\mu_0(\mathcal C_0) = 1$.

  By Chernoff, with probability $1 - \exp(-\frac{\mu(\mathcal C) m}{16}) \ge 1 - \frac{\delta}2$, at least
  $m_0 = \lfloor \frac{\mu(\mathcal C) m}4 \rfloor$ of the samples $\bm{z}_i$ are in $\mathcal C_0$.
  Without loss of generality, let $\bm{z}_1, \ldots, \bm{z}_{m_0}$ be those samples.
  Then
  \begin{align}
    \mathcal R_{\mathrm{Adv}, 2\rho} \left(f, \mu \right)
    &\ge \frac12 \mu(\mathcal C)~ \PP_{\bm{x} \sim \mu, \bm{x} \in \mathcal C_0} \left[ \exists \bm{z} \in \mathcal{B}_{2\rho}(\bm{x}), ~ f^*(\bm{x}) \neq f(\bm{z})) \right]
    \\ &= \frac12 \mu(\mathcal C)~ \PP_{\bm{x} \sim \mu_0} \left[ \exists \bm{z} \in \mathcal{B}_{2\rho}(\bm{x}), ~ f(\bm{z}) \neq 0 \right]
    \\ &\ge \frac12 \mu(\mathcal C)~ \PP_{\bm{x} \sim \mu_0} \left[ \exists~ i \le m_0 ~:~ \bm{z}_i \in \mathcal{B}_{2\rho}(\bm{x}) \cap \mathcal C_0,~ f(\bm{z}_i) \neq 0 \right]
    \\ &= \frac12 \mu(\mathcal C)~ \PP_{\bm{x} \sim \mu_0} \left[ \exists~ i \le m_0 ~:~ \bm{x} \in \mathcal{B}_{2\rho}(\bm{z_i}),~ \bm{z}_i \in \mathcal C_0,~ f(\bm{z}_i) \neq 0 \right].
    \\ &= \frac12 \mu(\mathcal C)~ \mu_0 \left( \bigcup_{i \le m_0,~ f(\bm{z}_i) \neq 0} \mathcal{B}_{2\rho}(\bm{z_i}) \right).
  \end{align}

  Let $\bm{s}_1, \ldots, \bm{s}_N$ be the centers of a minimum $\rho$-covering of $\mathcal C_0$. 

  The plan is the following: we will lower bound
  $\bigcup_{i \le m_0,~ f(\bm{z}_i) \neq 0} \mathcal{B}_{2 \rho}(\bm{z}_i)$
  by the union of some $\mathcal{B}_{\rho}(\bm{s}_k)$,
  which will have large $\mu_0$-measure in total.
  Moreover, each of the chosen $\mathcal{B}_{\rho}(\bm{s}_k)$ will have large
  enough $\mu_0$-measure. For this, we use the following general lemma:

  \begin{lemma}
    \label{lemma:greedy-subcover}
    Let $\bm{s}_1, \ldots, \bm{s}_N$ be the centers of some balls $\mathcal{B}_{r}(\bm{s}_i)$ in $\RR^d$,
    and take any measure $\nu$. Then, for any constant $0 < \alpha < 1$,
    there exists a subset $S \subseteq \set{1, \ldots, N}$ of the balls such that:
    \begin{itemize}
      \item $\nu \left( \bigcup_{i \in S} \mathcal{B}_{r}(\bm{s}_i) \right) \ge (1 - \alpha)~ \nu \left( \bigcup_{i = 1}^{N} \mathcal{B}_{r}(\bm{s}_i) \right)$.
      \item $\nu \left( \mathcal{B}_{r}(\bm{s}_i) \right) \ge \frac{\alpha}{N}~ \nu \left( \bigcup_{i = 1}^{N} \mathcal{B}_{r}(\bm{s}_i) \right)$ for all $i \in S$.
    \end{itemize}
  \end{lemma}
  Informally, the first condition says that the union of the chosen subset has a constant fraction of the measure of the union.
  The second condition says that each of the chosen balls has $\Omega(1/N)$ of the measure of the union of all balls.
  \begin{subproof}
    Without loss of generality, let the balls be ordered by measure:
    \begin{align}
      \nu \left( \mathcal{B}_{r}(\bm{s}_1) \right) \ge \nu \left( \mathcal{B}_{r}(\bm{s}_2) \right) \ge \cdots \ge \nu \left( \mathcal{B}_{r}(\bm{s}_N) \right).
    \end{align}
    We take the greedy subset $S = \set{1, \ldots, K}$, where $1 \le K \le N$ is the largest index such that 
    $\nu \left( \mathcal{B}_{r}(\bm{s}_K) \right) \ge \alpha~ \nu \left( \bigcup_{i = 1}^{N} \mathcal{B}_{r}(\bm{s}_i) \right)$. 
    \begin{align}
      \nu \left( \bigcup_{i = 1}^{K} \mathcal{B}_{r}(\bm{s}_i) \right)
      &= \nu \left( \bigcup_{i = 1}^{N} \mathcal{B}_{r}(\bm{s}_i) \right) - \nu \left( \bigcup_{i = K + 1}^{N} \mathcal{B}_{r}(\bm{s}_i) \right)
     \\ &\ge \nu \left( \bigcup_{i = 1}^{N} \mathcal{B}_{r}(\bm{s}_i) \right) - \left( N - K \right) \frac{\alpha}{N}~ \nu \left( \bigcup_{i = 1}^{N} \mathcal{B}_{r}(\bm{s}_i) \right)
      \\ &\ge \left(1 - \alpha \right)~ \nu \left( \bigcup_{i = 1}^{N} \mathcal{B}_{r}(\bm{s}_i) \right).
    \end{align}
    The first inequality follows because for all $i \ge K+1$ it holds $\nu \left( \mathcal{B}_{r}(\bm{s}_K) \right) < \alpha~ \nu \left( \bigcup_{i = 1}^{N} \mathcal{B}_{r}(\bm{s}_i) \right)$,
    and the second is because $\frac{N-K}{N} \le 1$.
  \end{subproof}

  We can apply the above to the situation in the proof of Theorem~\ref{thm:adv-risk-lower-bound} with $\alpha = \frac12$.
  Without loss of generality, order the covering $\bm{s}_1, \ldots, \bm{s}_N$ by the $\mu_0$-measure of the corresponding balls:
  \begin{align}
    \mu_0 \left( \mathcal{B}_{\rho}(\bm{s}_1) \right) \ge \mu_0 \left( \mathcal{B}_{\rho}(\bm{s}_2) \right) \ge \cdots \ge \mu_0 \left( \mathcal{B}_{\rho}(\bm{s}_N) \right).
  \end{align}

  \begin{corollary}
    \label{corr:greedy-subcover-covering}
    If $1 \le K \le N$ is the largest index such that $\mu_0(\mathcal{B}_\rho(\bm{s}_K) \ge \frac1{2N}$,
    then 
    \begin{align}
      \mu_0\left( \bigcup_{k=1}^K \mathcal{B}_{\rho}(\bm{s}_i) \right) > \frac12.
    \end{align}
  \end{corollary}

  \begin{figure}
    \centering
    \vspace{-3mm}
    \centering
    \begin{tabular}{cc}
      
      \includegraphics[width=0.4 \linewidth]{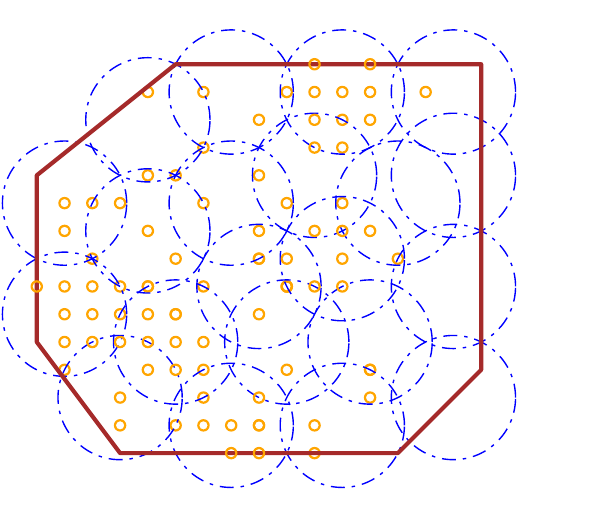} &
                                                                                   \includegraphics[width=0.4\linewidth]{figures/second_iter_proof_13.pdf} \\ 
      a) The original cover of $\mathcal C$.
                                                                                 & b) The dense greedy subcover. \\
    \end{tabular}
    \caption{
      Illustration of \Cref{lemma:greedy-subcover} and \Cref{corr:greedy-subcover-covering}.
      Given a cover of $N$ balls, we can pick a subcover of balls covering at least half of the measure, 
      with each ball having measure at least $\frac1{2N}$.
    }

    \label{fig:illustration-lemma-1}
  \end{figure}

  We now show that the chosen balls are dense enough to get samples
  in the training set with high probability.

  \begin{lemma}
    \label{lemma:appear-balls}
    With probability $1 - \delta/2$,
    each $\mathcal{B}_\rho(\bm{s}_k)$ for $k \le K$ contains at least one $\bm{z}_i \in \mathcal C_0$ such that $f(\bm{z}_i) \neq 0$.
  \end{lemma}

  \begin{subproof}
    We have
    \begin{align}
      \PP \left[ \bm{z}_i \in \mathcal{B}_\rho(\bm{s}_k) \mid \bm{z}_i \in \mathcal C_0 \right] \ge \frac1{2N},
    \end{align}
    and because the label corruption is independent from everything, we also have
    \begin{align}
      &\PP \left[ f(\bm{z}_i) \neq 0 \mid \bm{z}_i \in \mathcal C_0 \right] = \eta \\
      \implies 
      &\PP \left[ f(\bm{z}_i) \neq 0 ~\text{and}~ \bm{z}_i \in \mathcal{B}_\rho(\bm{s}_k) \mid \bm{z}_i \in \mathcal C_0 \right] \ge \frac{\eta}{2N} 
    \end{align}

    Therefore,
    \begin{align}
      &\PP \left[ \mathcal{B}_\rho(\bm{s}_k) \cap \set{\bm{z}_i : i \le m_0,~ f(\bm{z}_i) \neq 0}  = \emptyset \right]
   \\ &= \prod_{i=1}^{m_0} \PP \left[ z_i \notin \mathcal{B}_\rho(\bm{s}_k) ~\text{or}~ z_i \notin \mathcal C_0 ~\text{or}~ f(\bm{z}_i) \neq 0 \right]
      \\ &\le \left( 1 - \frac{\eta}{2N} \right)^{m_0} 
      \\ &\le \exp(-\frac{m_0 \eta}{2N}) \ge \frac{\delta}{2N},
    \end{align}
    In the last line we used $1 + x \le e^x$, the fact that $m_0 = \frac14 \mu(C) m$, and the expression for $m$ from \Cref{thm:adv-risk-lower-bound}.
    Now we can use the union bound to prove the lemma:
    \begin{align}
      \PP \left[ ~\text{some}~ \mathcal{B}_\rho(\bm{s}_k) ~\text{for}~ k \le K ~\text{contains no}~ \bm{z}_i : i \le m_0,~ f(\bm{z}_i) \neq 0~ \right] \le K \frac{\delta}{N} \le \delta.
    \end{align}
  \end{subproof}

  Finally, using both \Cref{lemma:greedy-subcover} and \Cref{lemma:appear-balls}, we can finish:
  \begin{align}
    \mathcal R_{\mathrm{Adv}, 2\rho}(f, \mu)
    &\ge \frac12 \mu(\mathcal C)~ \mu_0 \left( \bigcup_{i \le m_0,~ f(\bm{z}_i) \neq 0} \mathcal{B}_{2\rho}(\bm{z_i}) \right).
    \\ &\ge \frac12 \mu(\mathcal C)~ \mu_0 \left( \bigcup_{k=1}^K \mathcal{B}_{\rho}(\bm{s}_k) \right) 
         \ge \frac14 \mu(\mathcal C) .
  \end{align}
\end{proof}

\section{Proof of \Cref{prop:sphere-label-noise}}
\label{app:proof-sphere-label-noise}

\spherelabelnoise*

\begin{proof}
  Let the $m = 1.01^d \le \exp(d/80)$ samples be $\bm{z}_1, \ldots, \bm{z}_m$ with labels 
  $y_1, \ldots, y_m \in \{0, 1\}$. Almost surely the $\bm{z}_i$ are distinct.
  Define the interpolating classifier $f : \RR^d \to \{0, 1\}$ as
  \begin{align}
    \label{eq:interpolating-classifier}
    f(\bm{x}) =  
    \begin{cases}
      y_i & \text{if } \bm{x} \in \{ \bm{z}_1, \ldots, \bm{z}_m \}; \\
      \one_{x_1 > \frac12 } & \text{otherwise.}
    \end{cases}
  \end{align}

  We want to show $f$ is robust. Draw $\bm{x} = (x_1, \ldots, x_d)$ uniformly on $\sph^{d-1}$.
  There are only two ways $\bm{x}$ can contribute to the adversarial risk $\mathcal R_{\mathrm{Adv}, \rho}(f, \mu)$:
  \begin{itemize}
    \item $\bm{x}$ is close to a training sample $\bm{z_i}$ with label noise;
    \item $\bm{x}$ is close to the ``decision boundary'' $x_1=\frac12$ of $\sph^{d-1}$.
  \end{itemize}
  Hence, remembering \Cref{eq:adversarial-risk},
  \begin{align}
    \mathcal R_{\mathrm{Adv}, \rho}(f, \mu) &\le 
    \PP \left[ \bm{x} \text{ is in a } \rho \text{-ball around at least one of the } \bm{z}_i \right]
    +
    \PP \left[ \frac12 -\rho \le x_1 \le \frac12 + \rho \right].
    \\ &\le 
    \PP \left[ \bm{x} \text{ is in a } \rho \text{-ball around at least one of the } \bm{z}_i \right]
    +
    \PP \left[ x_1 \ge \frac12 -\rho \right].
  \end{align}
  By the union bound,
  \begin{align}
    &\PP \left[ \bm{x} \text{ is in a } \rho \text{-ball around at least one of the } \bm{z}_i \right]
    \\ &\le m ~ \PP \left[ \norm{\bm{x} - \bm{z}_1}_2 \le \rho \right]
    \\ &\le m ~ \PP \left[ \norm{\bm{x}}^2 + \norm{\bm{z}_1}^2 - 2 \langle \bm{x}, \bm{z}_1 \rangle \le \rho^2 \right]
    \\ &= m ~ \PP \left[ \langle \bm{x}, \bm{z}_1 \rangle \ge 1 - \rho^2 / 2 \right].
  \end{align}

  As $\mu$ is rotationally invariant, $\langle \bm{x}, \bm{z}_1 \rangle$ is distributed the same as $x_1$. We have proved
  \begin{align}
    \label{eq:adv-risk-bound-x1}
    \mathcal R_{\mathrm{Adv}, \rho}(f, \mu) 
    &\le 
    m ~ \PP \left[ x_1  \ge 1 - \frac{\rho^2}2 \right]
    +
    \PP \left[ x_1 \ge \frac12 -\rho \right].
  \end{align}

  We can bound $\PP[x_1 \ge t]$ for $t > 0$ as follows: let $g_1, \ldots, g_d$ be i.i.d. standard $N(0, 1)$ random variables.
  \begin{subequations}
    \begin{align}
      \PP \left[x_1 \ge t  \right]
         &= \PP \left[ \frac{g_1}{\sqrt{g_1^2 + \ldots + g_d^2}} \ge t \right]
      \\ &= \PP \left[ g_1^2 \ge t^2 (g_1^2 + \ldots + g_d^2) \right]
      \\ &= \PP \left[ \frac{1 - t^2}{t^2} g_1^2 \ge g_2^2 + \ldots + g_d^2 \right]
      \\ &\le \PP \left[ \frac{1 - t^2}{t^2} g_1^2 \ge \frac{d-1}2 \right] + \PP \left[ g_2^2 + \ldots + g_d^2 \le \frac{d-1}2 \right],
      \label{eq:reduce-to-laurent-massart}
    \end{align}
  \end{subequations}
  where the last inequality is because $a \ge b$ implies $a \ge c$ or $b \le c$.

  As $0 < \rho < \frac14$, we can take $t = \frac14$ in both probabilities in \Cref{eq:adv-risk-bound-x1}.
  We now use the often-cited chi-square bounds from Lemma 1 in \citet{laurent2000massart}.
  \begin{subequations}
    \begin{align}
      \label{eq:laurent-massart-upper}
      \PP \left[ g_2^2 + \ldots + g_d^2 \le (d-1) - 2 \sqrt{(d-1)s} \right] \le \exp \left( -s \right) \\
      \label{eq:laurent-massart-lower}
      \PP \left[ g_1^2 \ge 1 + 2 \sqrt{s} + 2s \right] \le \exp \left( -s \right)
    \end{align}
  \end{subequations}

  Then for $s = \frac{d}{40}$, it's easy to see that both probabilities in \Cref{eq:reduce-to-laurent-massart}
  are less than the corresponding probabilities in \Cref{eq:laurent-massart-lower} and \Cref{eq:laurent-massart-upper}.

  Finally, as $d$ goes to infinity,
  \begin{align}
    \label{eq:adv-risk-final}
    \mathcal R_{\mathrm{Adv}, \rho}(f, \mu) 
    &\le 
    m \exp(-d/40) + \exp(-d/40) \\
    &\le
    \exp(-d/80) + \exp(-d/40) \to 0.
  \end{align}
\end{proof}

\section{Poisoning theorem}
\label{app:poisoning-theorem}

Recall the definition of the adversarial risk $\mathcal{R}_{\mathrm{Adv}, \rho}$ from \Cref{sec:main-theoretical-results}:
\begin{align}
 \label{eq:adversarial-risk-temp}
  \mathcal R_{\mathrm{Adv}, \rho}(f, \mu) = \PP_{\bm{x} \sim \mu} \left[ \exists \bm{z} \in \mathcal{B}_{\rho}(\bm{x}), ~ f^*(\bm{x}) \neq f(\bm{z})) \right].
\end{align}

For interpolating classifiers with minimal test error, there are two sources of adversarial risk: decision boundaries and label noise.
In this paper, we are specifically interested in the latter.
The theorem %
is easiest to formalize in the case where the decision boundary contribution to the adversarial risk is negligible. 
For example, this is the case when the classes are separable with a large margin, or in real-world datasets when there is not much data which humans would label ambiguously.

Therefore, instead of working with the adversarial risk, we introduce the \emph{separable proxy} for the adversarial risk.
Let the label noised points be $\mathcal{S} = \{\bm{s}_1, \ldots, \bm{s}_N \}$, and let $f$ interpolate the training set, and otherwise
minimize the test error.
\begin{align}
 \label{eq:proxy-adversarial-risk}
  \Rproxy{\rho}(\mathcal{S}, \mu) \defeq \PP_{\bm{x} \sim \mu} \left[ \exists 1 \le k \le N,~  \bm{s}_k \in \mathcal{B}_{\rho}(\bm{x}) \right] 
  \le 
  \mathcal R_{\mathrm{Adv}, \rho}(f, \mu).
\end{align}

The proxy adversarial risk $\Rpoison$ for the poisoned case is easy to define:
\begin{align}
  \label{eq:proxy-adversarial-risk-poison}
  \Rpoison_{\rho, N} 
  &= \sup_{\bm{s}_1, \ldots, \bm{s}_N \in \supp(\mu)} \Rproxy{\rho}(\{\bm{s}_1, \ldots, \bm{s}_N\}, \mu)
  \\ &= \sup_{\bm{s}_1, \ldots, \bm{s}_N \in \supp(\mu)} \mu \left( \bigcup_{k=1}^N \mathcal{B}_{\rho}(\bm{s}_k) \right).
\end{align}

In fact, by a compactness argument, the $\sup$ above can be replaced by $\max$, but this is not important for the theorem we want to prove.

The uniform version $\Runif$ is a random variable. Its value depends on the random training set and which points get their labels flipped.
We define it as follows:
\\
Sample a training set of $T$ random points from $\mu$ independently, and let $\mathcal{S}$ be a random subset of the training set,
with each point taken with probability $\eta$.
Then 
\begin{align}
  \Runif_{\rho, T, \eta} := \Rproxy{\rho}(\mathcal{S}, \mu).
\end{align}

We are now ready to state our formal theorem.
\begin{theorem}
  \label{thm:uniform-vs-poisoning}
  Let $\eta > 0$ be the label noise rate, 
  and let $T, m$ be positive integers representing the training set sizes for Uniform and Poisoner.
  Fix $N = \lfloor \eta m \rfloor$ to be be the number of labels the poisoner can flip.
  For $\delta > 0$, with probability $1 - \delta$, for any $\rho > 0$,
  \begin{align}
    \Runif_{2\rho, T, \eta} \geq \frac12 \Rpoison_{\rho , N}
  \end{align}
  for $T = \Omega \left(\frac{m (\log m + \log \frac1\delta) }{\Rpoison_{\rho, N}} \right)$.
\end{theorem}

To be precise, the Uniform adversary takes at least the following number of samples:
\begin{align}
  \label{eq:exact-T}
  T 
  = \frac{2N}{\eta \Rpoison_{\rho, N}} \left( \log N + \log \frac1\delta \right)
  = \frac{2m}{\Rpoison_{\rho, N}} \left( \log m + \log \eta + \log \frac1\delta \right).
\end{align}
The proof is quite similar in spirit to the proof of \Cref{thm:adv-risk-lower-bound}.
\begin{proof}
  Let the Poisoner pick points $\bm{s}_1, \ldots, \bm{s}_N \in \RR^d$.
  We want to prove that with high probability,
  \begin{align*}
    \Runif_{2\rho, T, \eta} \geq \frac12~ \mu \left( \bigcup_{k=1}^N \mathcal{B}_{\rho}(\bm{s}_k) \right).
  \end{align*}

  We show that, with high probability, a sample of $T$ points from $\mu$, 
  with uniform label noise with probability $\eta$, will result in many of the balls $\mathcal{B}_{\rho}(\bm{s}_k)$ having a mislabeled point in them.

  Order the points $\bm{s}_1, \ldots, \bm{s}_N$ such that
  \begin{align}
    \mu \left( \mathcal{B}_{\rho}(\bm{s}_1) \right) \ge \mu \left( \mathcal{B}_{\rho}(\bm{s}_2) \right) \ge \cdots \ge \mu \left( \mathcal{B}_{\rho}(\bm{s}_N) \right).
  \end{align}

  By \Cref{lemma:greedy-subcover} with $\alpha = \frac12$, we know that there exists $K$ such that
  \begin{align}
    \label{eq:poisoning-proof-upper-bound-rpoison}
    \mu \left( \bigcup_{k=1}^K \mathcal{B}_{\rho}(\bm{s}_k) \right) \ge \frac12~\mu \left( \bigcup_{k=1}^N \mathcal{B}_{\rho}(\bm{s}_k) \right) 
    = \frac12 \Rpoison_{\rho, N}
  \end{align}
  and $\mu \left( \mathcal{B}_{\rho}(\bm{s}_k) \right) \ge \frac1{2N} \Rpoison_{\rho, N}$ for all $1 \le k \le K$.

  As in \Cref{app:proof-of-improved-thm1}, we proceed to show that with high probability,
  a random sample of $T$ points from $\mu$ will hit each of the balls $\mathcal{B}(\bm{s}_k, \rho)$ for $k \le K$,
  because the chosen balls are dense enough to get hit when $T \gtrsim m \log m$.
  This is enough to prove
  \begin{align}
    \label{eq:poisoning-proof-lower-bound-runif}
    \Runif_{2\rho, T, \eta} \ge \mu \left( \bigcup_{k=1}^K \mathcal{B}_{\rho}(\bm{s}_k) \right), 
  \end{align}
  since any two points in the same $\mathcal{B}_\rho(\bm{s}_k)$ are within distance $2\rho$ of each other.

  Let $\bm{z}_1, \ldots, \bm{z}_T$ be a random dataset of $T$ points from $\mu$.
  We have a lemma similar to \Cref{lemma:appear-balls}:
  \begin{lemma}
    \label{lemma:appear-balls-poisoning}
    With probability $1 - \delta$,
    each $\mathcal{B}_\rho(\bm{s}_k)$ for $k \le K$ contains a mislabeled point $\bm{z}_i$.
  \end{lemma}

  \begin{subproof}
    We have
    \begin{align}
      \PP \left[ \bm{z}_i \in \mathcal{B}_\rho(\bm{s}_k) \right] \ge \frac{1}{2N} \Rpoison_{\rho, N} 
    \end{align}
    and because the label corruption is independent from everything, we also have
    \begin{align}
    \PP \left[ \bm{z}_i \in \mathcal{B}_\rho(\bm{s}_k) ~\text{and}~ \bm{z}_i ~\text{mislabeled} \right] \ge \frac{\eta}{2N} \Rpoison_{\rho, N}.
    \end{align}

    Therefore,
    \begin{align}
      &\PP \left[ \mathcal{B}_\rho(\bm{s}_k) \cap \set{ \text{mislabeled}~\bm{z}_i }  = \emptyset \right]
   \\ &= \prod_{i=1}^{T} \PP \left[ z_i \notin \mathcal{B}_\rho(\bm{s}_k) ~\text{or}~ \bm{z}_i ~\text{mislabeled} \right]
   \\ &\le \left( 1 - \frac{\eta}{2N} \Rpoison_{\rho, N} \right)^T
   \\ &\le \exp(-\frac{T \eta}{2N} \Rpoison_{\rho, N}) \le \frac{\delta}{N}.
    \end{align}

    Here in the last line we used $1 + x \le e^x$ and \Cref{eq:exact-T}.

    Now we can use the union bound to prove the lemma:
    \begin{align}
      \PP \left[ ~\text{some}~ \mathcal{B}_\rho(\bm{s}_k) ~\text{for}~ k \le K ~\text{contains no mislabeled}~ \bm{z}_i~ \right] \le K \frac{\delta}{N} \le \delta.
    \end{align}
  \end{subproof}

  Combining \Cref{eq:poisoning-proof-upper-bound-rpoison} and \Cref{eq:poisoning-proof-lower-bound-runif}, we have the desired
  \begin{align}
    \Runif_{2\rho, T, \eta} \geq \frac12 \Rpoison_{\rho , N}.
  \end{align}

  As this holds for any $\bm{s}_1, \ldots, \bm{s}_N$, and the left-hand-side does not depend on the chosen points $\bm{s}_k$,
  we can take the supremum over all possible choices of $\bm{s}_1, \ldots, \bm{s}_N$ to prove the theorem.
\end{proof}

\section{Proof of~\Cref{prop:benign-label-noise}}
\benignlabelnoise*
\begin{proof}
    To prove the adversarial risk for \(\mathcal{D}_1\), we simply
    invoke~\Cref{thm:adv-risk-lower-bound}. Consider the compact
    \(\mathcal{C}=\bigcup_{i=1}^A \left(i, i+\frac12 \right)\). The
    covering number for \(\mathcal{C}\) is \(N=\frac{A}{2\rho}\) and
    its probability mass \(\mu\left(\mathcal{C}\right) =
    \frac{1}{2}\).  By~\Cref{thm:adv-risk-lower-bound}, for \(m\geq
    \frac{16A}{\rho \eta}\log\left(\frac{A}{\rho\delta}\right)\), with
    probability greater than \(1-\delta\) we have that $\mathcal
    R_{\mathrm{Adv}, \rho}(f, \mu) \geq \frac{1}{8}$. This proves
    the first part.

    For the second part, using Hoeffding's inequality, as long as
    \(m\geq 16\log\left(\frac{2}{\delta}\right)\), we have that with
    probability at least \(1-\delta\), the number of  samples in
    \(\left[A, A+B\right]\) is less than \(\frac{3 m}{4}\). Therefore,
    the number of mislabeled samples in that interval is also less
    than \(\frac{3m}{4}\)with the same probability. Now consider the
    interpolator that is zero everywhere except at the mislabeled
    points. The maximum adversarial risk of this interpolator is the
    probability mass of the union of the intervals \([A+j,
    A+j+\frac{1}{2}]\) each of the mislabeled points lie in. This
    probability mass is at most
    \(\frac{3m\rho}{8B}\). Setting, \(m=\frac{16A}{\rho \eta}\log\left(\frac{A}{\rho\delta}\right)\), we obtain \(\mathcal
    R_{\mathrm{Adv}, \rho}(f, \mu)=\widetilde{O}_\rho\left(\frac{A}{ B}\right)\).
\end{proof}

\section{Proof of inductive bias}
\label{app:ind-bias}

\thmrepreparinter*

  \begin{proof}
    For any \(\rho\geq 0,~W\gg\rho\), construct a distribution \(\mu\)
    on \([0,W]\times\{0\}\) as follows. Distribute the covariates
    uniformly randomly in
    \([0,~\frac{W}{2}-2\rho]\bigcup~[\frac{W}{2}+2\rho,~W]\) and then
    label then with the ground truth labelling
    function~\(f^*\br{x}=\mathbbm{1}\{x_1\geq \frac{W}{2}\}\) where
    \(x=[x_1, x_2]\) is the two-dimensional covariate. Next, we
    construct an \(m\) dimensional dataset and flip each label
    independently with probability \(1-\eta\). We denote this set
    with~\(S_{m,\eta}\).

    The hypothesis class \(\SF\) is the class of one-dimensional
    thresholds on the first coordinate of the input space~(ignores the
    second coordinate entirely).  Define the following interpolating
    classifier \(f\in\SF:\reals^2\rightarrow \{0,1\}\) as follows
    \[f(x) = \begin{cases}y_1&~\text{if~\(x\)~is in}~S_{m,\eta}\\
    \mathbbm{1}\{x_1\geq
    \frac{W}{2}\}&\text{otherwise}\end{cases}.\]
As the sampling of the covariates and the label noise are independent events,
    \[\EE_{S_{m,\eta}}\left[\text{\# of mislabeled points
    in}~S_{m,\eta}\right] =  m\eta.\] Then the expected measure of the
    set of points adversarially vulnerable by an adversary of
    perturbation magnitude \(\rho\) on the classifier \(h\), as
    defined above, is upper bounded by \(2\rho m\eta\). Using the fact
    that the total measure of the domain is \(W\) and that \(m=\Theta\br{\frac{1}{\eta}}\), we get
    that\[\EE_{S_{m,\eta}}\left[\mathcal{R}_{\mathrm{Adv},\rho}\br{f; \mu}\right] \leq \frac{2\rho
    m\eta}{W} = \mathcal{O}(\rho).\]

  Next, consider the hypothesis class \(\SH\) defined as follows.
  Given a set of points \(\SZ=\{z_1,\ldots,z_k\}\in[0,W]^k\) and
  \(\gamma>\rho\), define the hypothesis
  \[h_{\SZ,\gamma}\br{x}=\begin{cases}1&\exists
  z\in\SZ~\vert~\mathbbm{1}\{x_2<\rho\}\wedge \mathbbm{1}\{x_1=z\}\\
  1&\exists
  z\in\SZ~\vert~\mathbbm{1}\{x_2<\rho\}\wedge
  \mathbbm{1}\{\abs{x_1-z}\leq \gamma\}\\
  0&\text{otherwise.} \end{cases}\] If \(\widetilde{S}\) is the set of mislabeled
\(1\)s in \(S_{m,\eta}\), then for any interpolating classifier
\(h_{\SZ,\gamma}\), it holds that \(\widetilde{S}\subseteq\SZ\). Next,
by construction, for every point \(z\in\SZ\), it holds that all points
\(x\in [z-\gamma,z+\gamma]\) can adversarially perturbed in the
\(x_2\) component to obtain the label \(1\). Thus the total measure of
the adversarially vulnerable set of points is greater than the number
of mislabeled points, whose original label is zero, multipled with
\(2\gamma\), which is \(2m\eta\gamma\).

Thus, we have that for any \(h\in\SH\) that interpolates
\(S_{m,\eta}\),\[\EE_{S_{m,\eta}}\left[\mathcal{R}_{\mathrm{Adv},\rho}\br{f;
\mu}\right] \geq \min\br{\frac{2\gamma m\eta}{W},~\frac{1}{2}} =
\Omega(\gamma).\]

  \end{proof}

We proved the adversarial risk bounds only in expectation over the training set;
but note that both of the bounds in~\Cref{thm:repre-par-inter} can be transformed into high
probability bounds using concentration inequalities. 

For simplicity, we did not treat the second bound in the theorem as a learning problem. 
However, it is possible to show that there exists a learning algorithm that uses
a similar number of samples to output \(f\in\SF\) such that
the adversarial risk is \(\mathcal{O}\br{\rho}\).

\begin{figure}[t]
    \centering
    \begin{tabular}{cc}
    \includegraphics[width=0.43\textwidth]{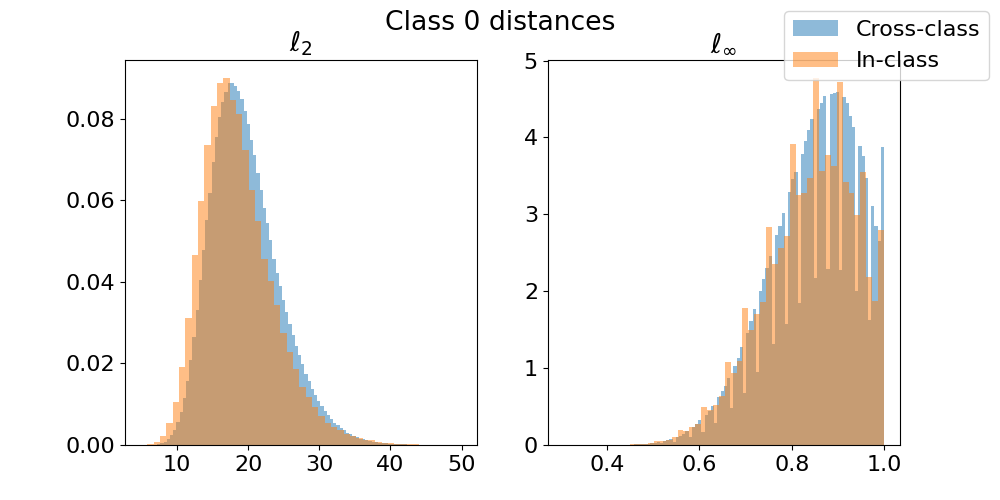} &
    \includegraphics[width=0.43\textwidth]{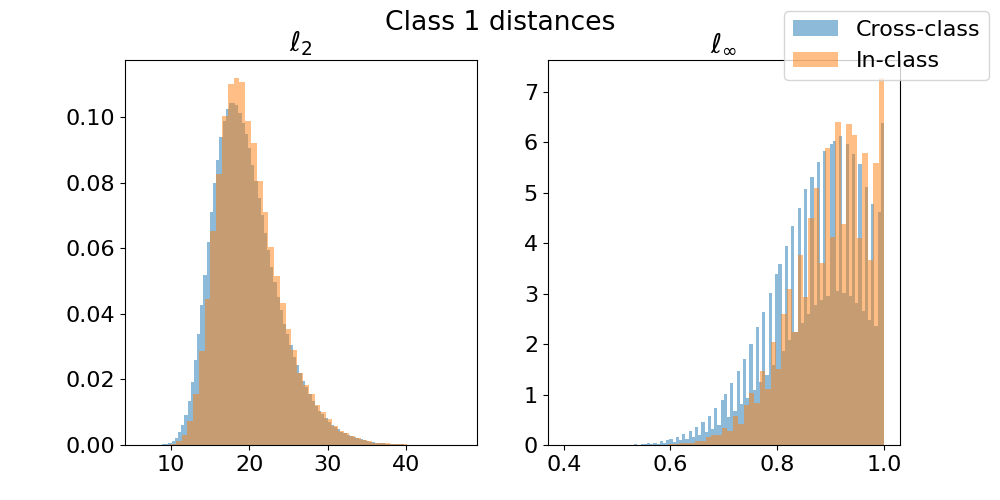} \\
    \includegraphics[width=0.43\textwidth]{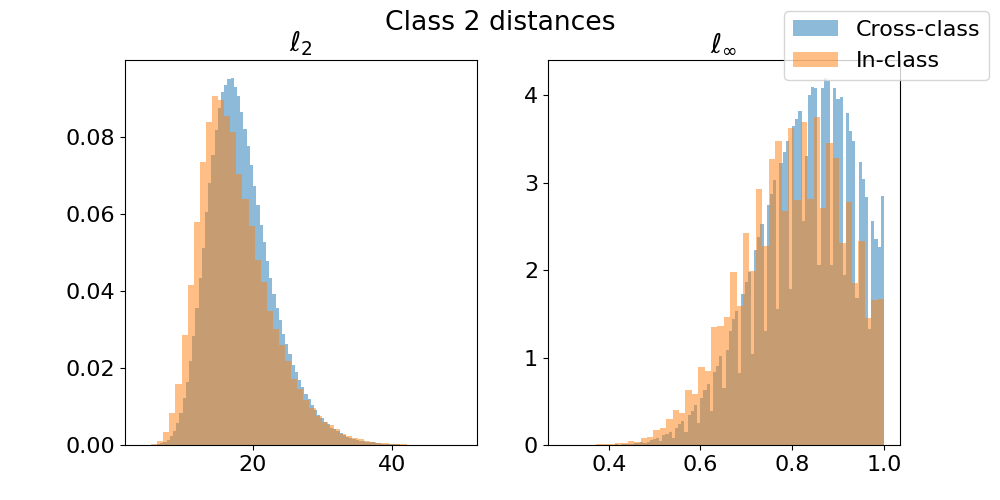} &
    \includegraphics[width=0.43\textwidth]{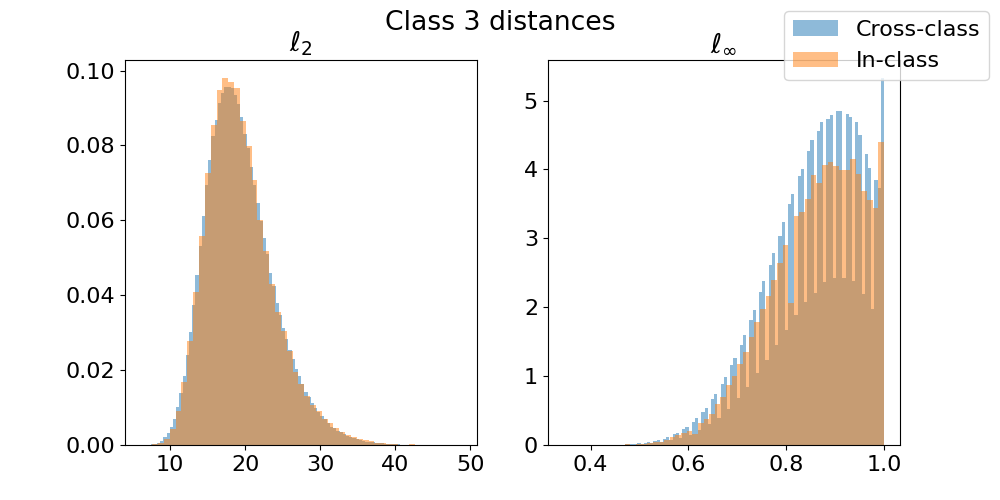} \\
    \includegraphics[width=0.43\textwidth]{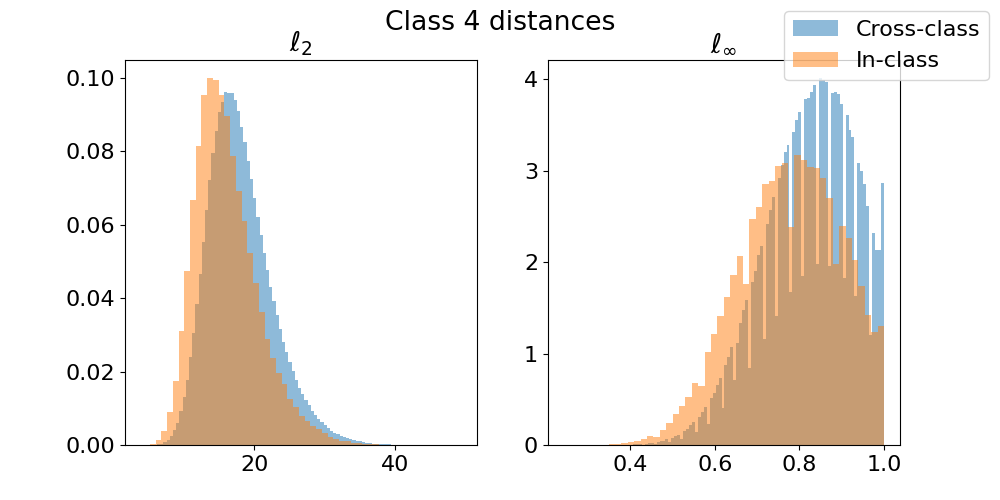} &
    \includegraphics[width=0.43\textwidth]{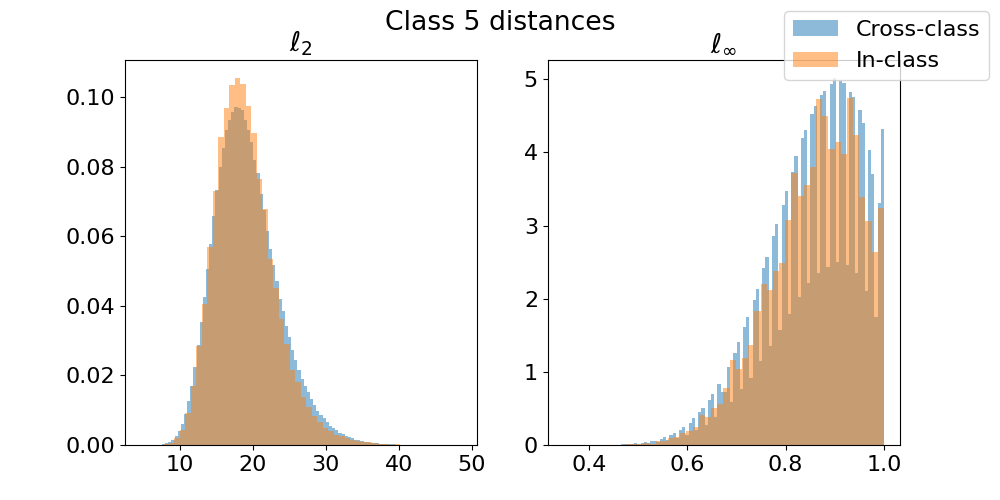} \\
    \includegraphics[width=0.43\textwidth]{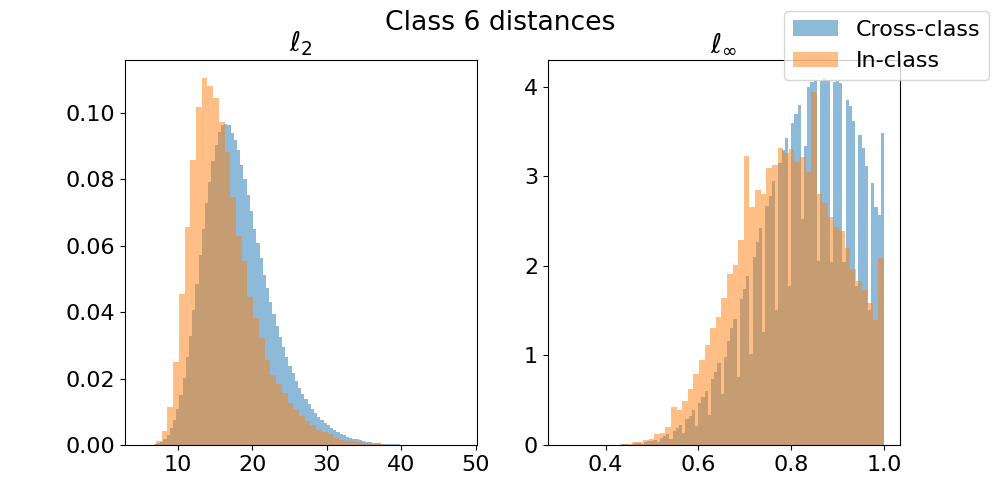} &
    \includegraphics[width=0.43\textwidth]{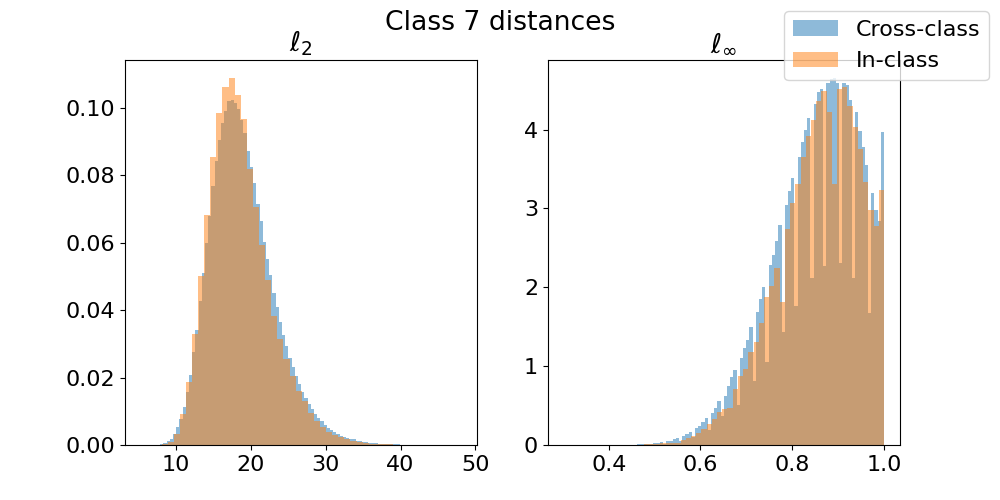} \\
    \includegraphics[width=0.43\textwidth]{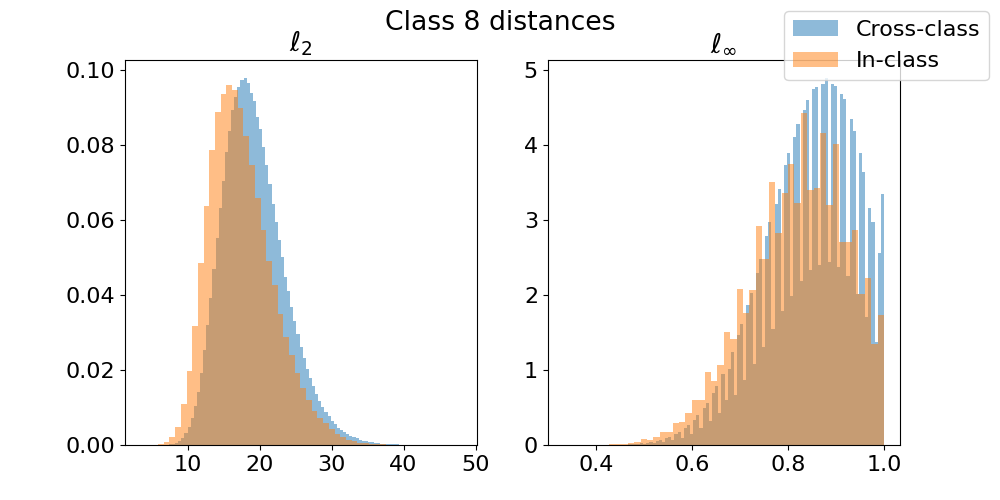} &
    \includegraphics[width=0.43\textwidth]{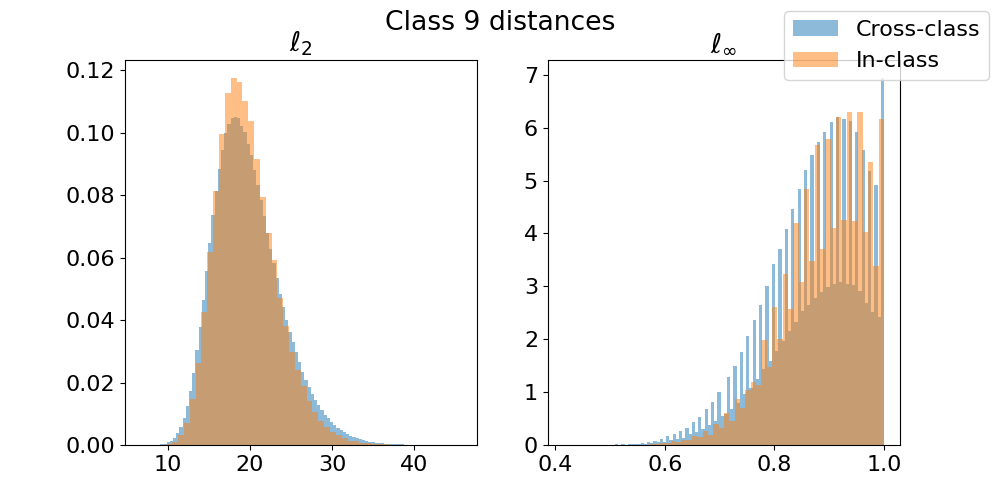} \\
    \end{tabular}
    
    \caption{Separation and average distances of different classes in CIFAR-10.}
    \label{fig:distances_cifar10_0.1_0}
\end{figure}

\section{Separation and average distances of different classes}
\label{app:distances}

For a classification dataset,
it is not easy to check if a given class can be covered by balls
of some radius in a given metric, such that each ball contains
only points from a single class.

However, we can estimate the distance distribution between points
inside a class and between points of different classes, 
If those two distributions are similar, and especially if the minimum
distances are roughly the same, then it seems likely that it's
not possible for the balls to contain only points from a single class.

The plots in \Cref{fig:distances_cifar10_0.1_0} strongly suggest that 
the points of different classes are not extremely ``far off'' in the $\norm{\cdot}_2$ 
or $\norm{\cdot}_\infty$ metrics, compared to the distances between points
inside a class.

\section{Experiments regarding inductive bias}
\label{app:exp-ind-bias}
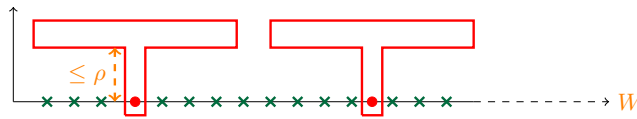
\begin{figure}[t]\centering
        \scalebox{0.9}{\begin{tikzpicture}
          \draw[] (1.2,0) -- (8,0);\draw[dashed, ->] (8,0) -- (10,0)
          node[right, orange]{\(W\)}; \filldraw [red] (3,0) circle
          (2pt);
          \filldraw [red] (6.5,0) circle (2pt);
          \draw[->] (1.2,0) -- (1.2,1+0.4);

          \draw[-, red, line width=1pt] (2.85,-0.2) --
          (3.15,-0.2) -- (3.15, 0.4+0.4) -- (4.5, 0.4+0.4) -- (4.5, 0.8+0.4) -- (1.5,0.8+0.4) -- (1.5, 0.4+0.4) -- (2.85, 0.4+0.4) -- (2.85, -0.2); 
          \draw[dashed, <->, orange, line width=1pt] (2.7, 0) -- node[midway,left] {\(\leq\rho\)}  (2.7, 0.8);
          \draw[-, red, line width=1pt] (2.85+3.5,-0.2) --
          (3.15+3.5,-0.2) -- (3.15+3.5, 0.4+0.4) -- (4.5+3.5, 0.4+0.4) -- (4.5+3.5, 0.8+0.4) --
          (1.5+3.5,0.8+0.4) -- (1.5+3.5, 0.4+0.4) -- (2.85+3.5, 0.4+0.4) -- (2.85+3.5, -0.2); 
          \draw (1.7,0 )
        node[cross,cadmiumgreen,line width=1pt] {}; \draw (2.1,0 )
        node[cross,cadmiumgreen,line width=1pt] {}; \draw (2.5,0 )
        node[cross,cadmiumgreen,line width=1pt] {}; \draw (3.4,0 )
        node[cross,cadmiumgreen,line width=1pt] {}; \draw (3.8,0 )
        node[cross,cadmiumgreen,line width=1pt] {}; \draw (4.2,0 )
        node[cross,cadmiumgreen,line width=1pt] {}; \draw (2.2+2.4,0 )
        node[cross,cadmiumgreen,line width=1pt] {}; \draw (2.6+2.4,0 )
        node[cross,cadmiumgreen,line width=1pt] {}; \draw (3.+2.4,0 )
        node[cross,cadmiumgreen,line width=1pt] {}; \draw (3.4+2.4,0 )
        node[cross,cadmiumgreen,line width=1pt] {}; \draw (3.4+2.4+0.4,0 )
        node[cross,cadmiumgreen,line width=1pt] {}; \draw (3.4+2.4+1,0 )
        node[cross,cadmiumgreen,line width=1pt] {}; \draw (3.4+2.4+1.4,0 )
        node[cross,cadmiumgreen,line width=1pt] {}; \draw (3.4+2.4+1.8,0 )
        node[cross,cadmiumgreen,line width=1pt] {};

      \end{tikzpicture}}%
            \caption{Visualization of a portion of the distribution
            \(\mu\) and the hypothesis class \(\SH\) used in~\Cref{thm:repre-par-inter}.~The crosses are the
            mislabeled examples and the circles are correctly labelled
            examples. All the circles are adversarially
            vulnerable to upwards perturbations of magnitude less than \(\rho\).} 
            \label{fig:t-shape-region-app}\vspace{-10pt}
          \end{figure}

We show that the the structure of the T-shaped classifier used
in~\Cref{thm:repre-par-inter} is visible when fitting neural networks to label noise on a tiny dataset.
We sample a three dimensional dataset of points $(X, Y, Z)$ with labels in $\{0, 1\}$ as follows:
\begin{itemize}
  \item
  $X$ is sampled uniformly from the segment $[0, 1]$;
  
  \item
  $Y$ is sampled from a normal distribution $\SN(0, 0.1)$;

  \item
  $Z$ is sampled from a normal distribution $\SN(0, 0.001)$;

  \item
  The label is 1 if $X > 0.5$, and 0 otherwise.

\end{itemize}

We sample $50$ points from this distribution to create the clean dataset. 
To create the noisy dataset, we randomly flip $10\%$ of the labels to generate the noisy
dataset. 

Then we train a one-hidden layer MLP with $1000$ hidden units
using the ADAM optimizer with a learning rate of $0.01$.  
The decision boundary after running this for $350$ epochs with a batch size of $20$
is plotted in~\Cref{fig:ind-bias-exp}. 
All our models interpolate the dataset~(both clean and noisy).

We plot the decision region in the $XY$ plane for three different values
of the Z dimension in~\Cref{fig:ind-bias-exp} for models trained on
the noisy dataset as well as the clean dataset. The first row in a box
corresponds to the model trained on the noisy dataset and the second row corresponds to the model trained on the clean dataset. 
The maroon circles inside the plots are balls of radius $0.04$, 
drawn around the points with label noise, indicating the region of adversarial
vulnerability induced by the points in that plane. 

As visualizations like these are often susceptible to variance due to random seeds, 
we report for three different seeds, denoted as Run 1, Run 2, and Run 3.

To interpret the T-like structure from~\Cref{fig:t-shape-region-app}
in these experiments, note that $\rho=0.04$, so the ``head'' of the
`T' is in the $XY$ plane for \(Z=0.04\). Further, the \(Z=-0.04\) is
essentially the head of the `T' for the other class. In the first rows
in all of the boxes~(i.e. the noisy dataset), note that ``heads of the
Ts'' are almost entirely within the decision region of one of the
classes. 
This indicates that all points on the $XY$ plane at~\(Z=0\) can
be perturbed along the $Z$-axis with a perturbation less than \(0.04\)
to change its predicted label, yielding an adversarial risk of 100\%. 
However, in the \(Z=0\) plane, the region of vulnerability is the union of the 
maroon balls, which is significantly smaller than what is what induced by perturbation
along the Z-dimension. 
This suggests that our intuition with the T-shaped classifiers 
may be relevant in practice for neural networks.

\begin{figure}[t]
\begin{framed}\centering
    \begin{subfigure}[t]{0.8\linewidth}
    \begin{subfigure}[t]{0.99\linewidth}
    \includegraphics[width=0.99\linewidth]{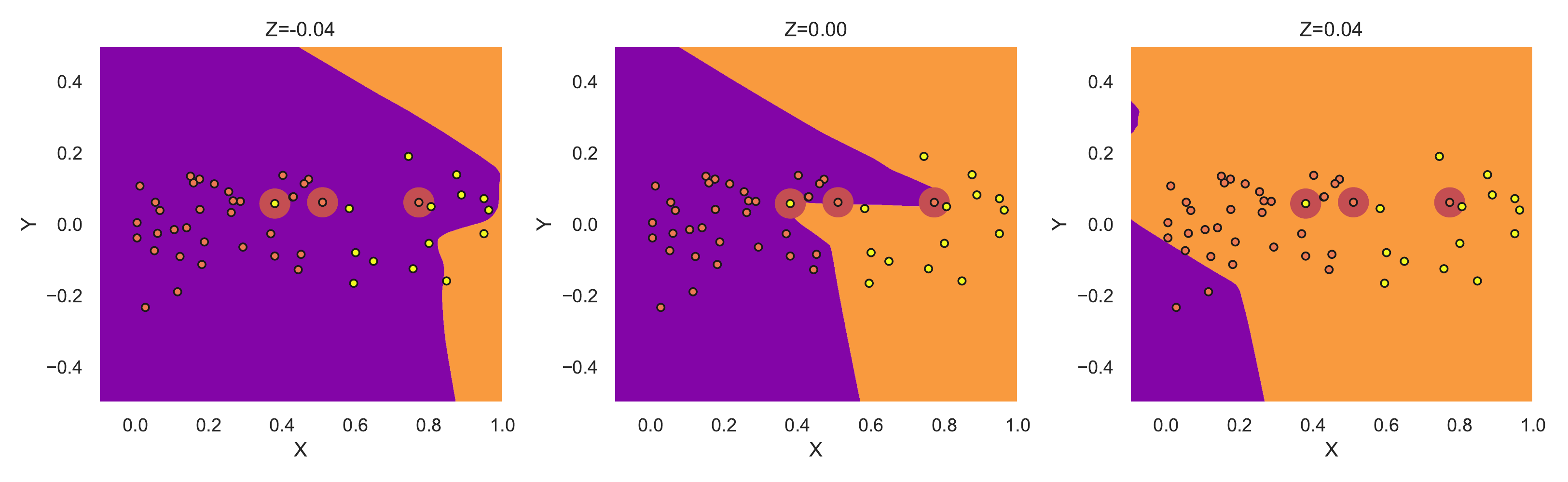}%
    \end{subfigure}\vspace{-2ex}
    \begin{subfigure}[t]{0.99\linewidth}
        \includegraphics[width=0.99\linewidth]{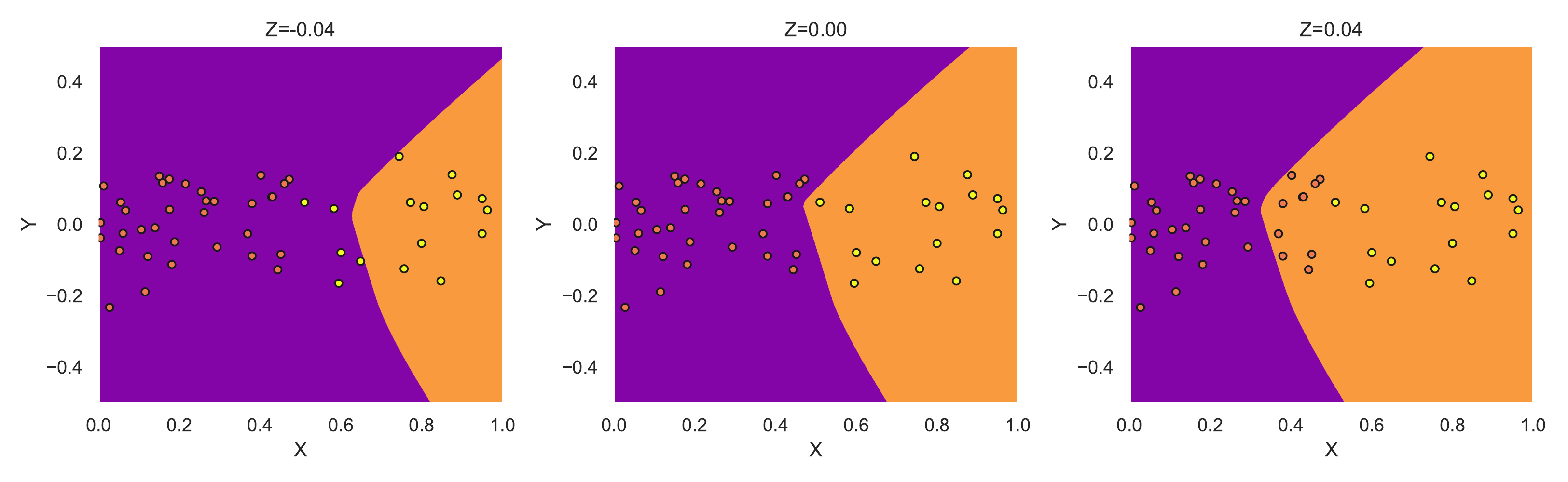}%
        \end{subfigure}\vspace{-1.5ex}
    \end{subfigure}
\end{framed}
\begin{framed}\centering
    \begin{subfigure}[t]{0.8\linewidth}
        \begin{subfigure}[t]{0.99\linewidth}
        \includegraphics[width=0.99\linewidth]{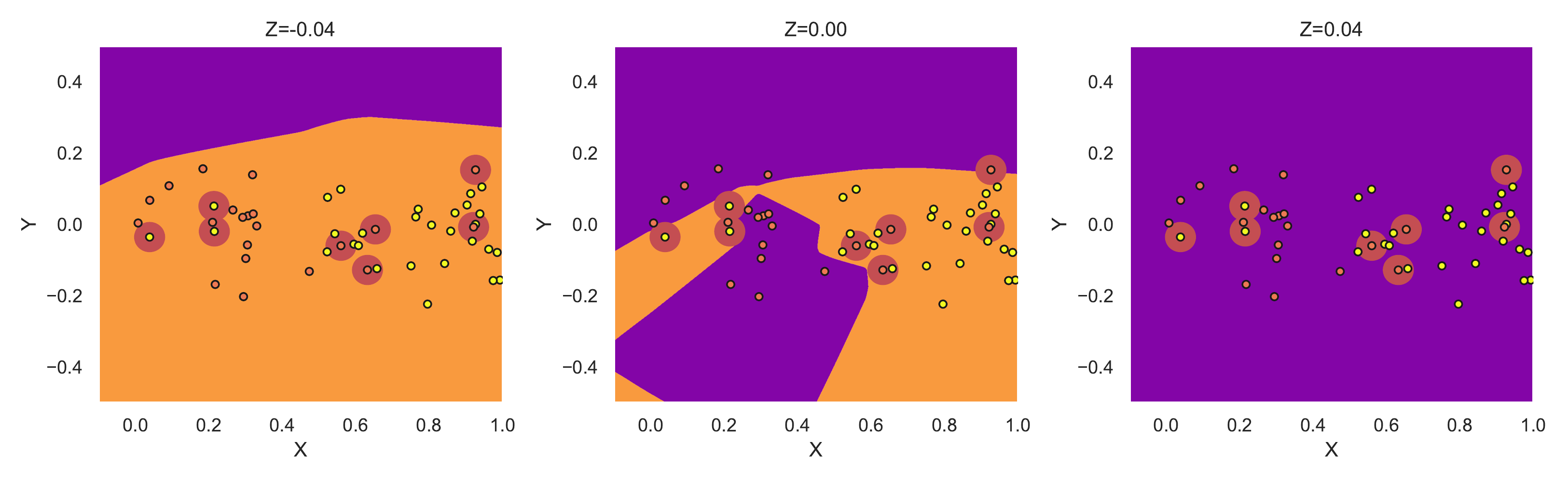}%
        \end{subfigure}\vspace{-2ex}
        \begin{subfigure}[t]{0.99\linewidth}
            \includegraphics[width=0.99\linewidth]{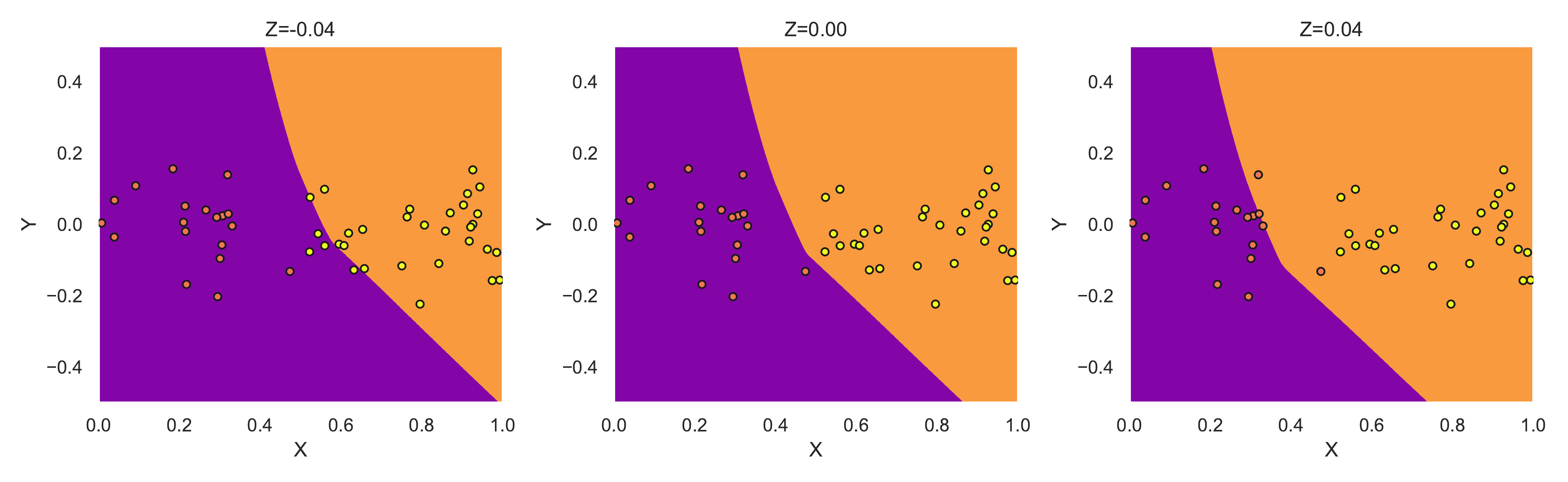}%
            \end{subfigure}\vspace{-1.5ex}
        \end{subfigure}
    \end{framed}
        \begin{framed}\centering
        \begin{subfigure}[t]{0.8\linewidth}
            \begin{subfigure}[t]{0.99\linewidth}
            \includegraphics[width=0.99\linewidth]{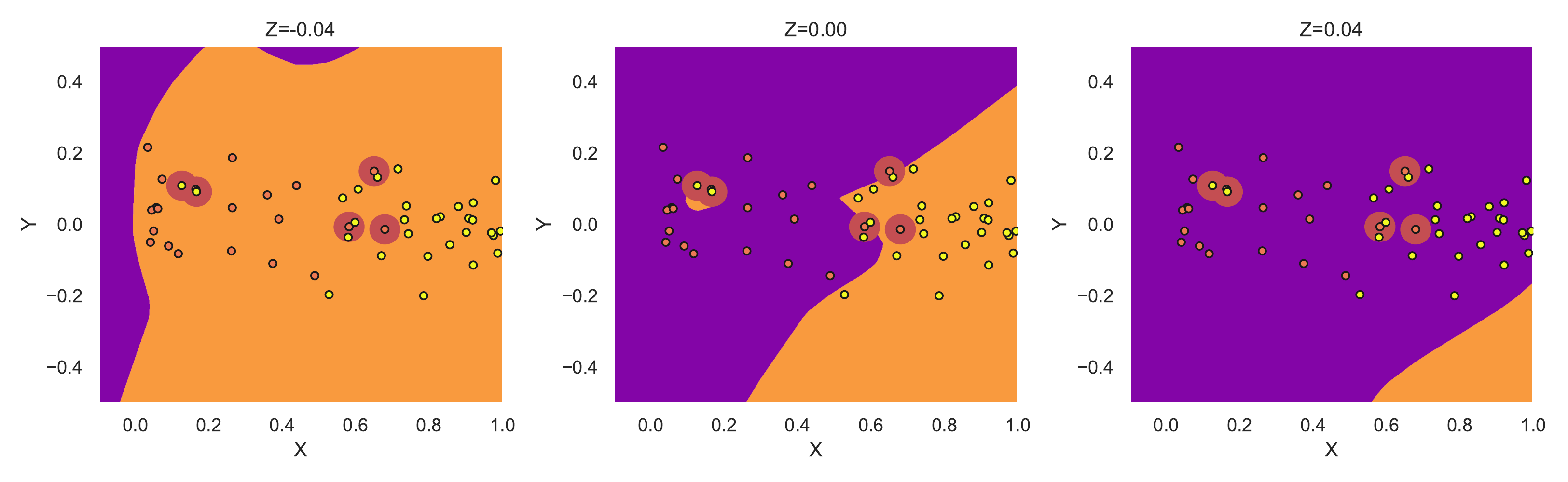}%
            \end{subfigure}\vspace{-2ex}
            \begin{subfigure}[t]{0.99\linewidth}
                \includegraphics[width=0.99\linewidth]{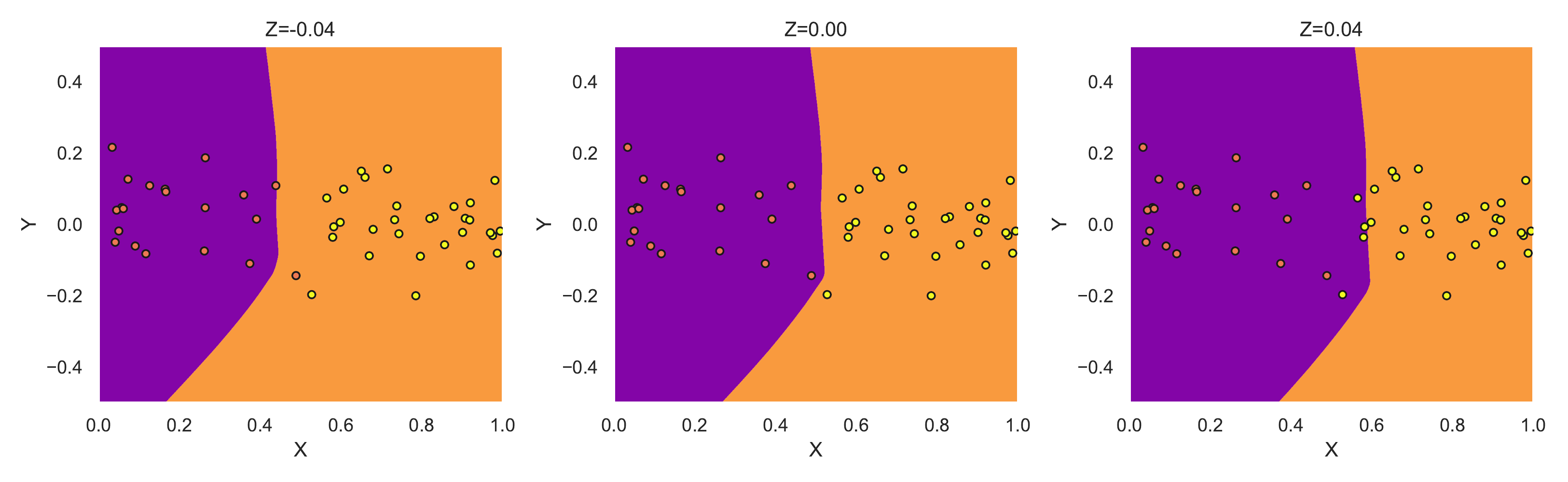}%
                \end{subfigure}\vspace{-1.5ex}
            \end{subfigure}
        \end{framed}
    \label{fig:ind-bias-exp}
    \caption{Each box is a different independent sample of the dataset. 
      \rebuttal{The first row in each box is with label noise, and the second row is without label noise.}
    The three plots in each row
    (\(Z\in\{-0.04,0.00,0.04\}\)) show the decision boundary of
    the interpolating model on the $XY$ plane for different values of $Z$. 
    The~\(Z=0.04\) can be interpreted as the head of the 'T' shaped
    decision region and~\(Z=-0.04\) is similarly an inverted 'T' for
    the other class.  The plots clearly show that when the model
    interpolates label noise, the width of the head of the 'T's are
    significantly more responsible for adversarial vulnerability than
    the decision region in the \(Z=0\) plane.}
\end{figure}

\end{document}